\documentclass[11pt]{article}

\usepackage{amssymb,amsmath,amsthm}
\usepackage{bbm}
\usepackage[noend]{algpseudocode}
\usepackage{graphicx}
\usepackage{verbatim}
\usepackage{url,xspace}
\usepackage{hyperref}
\usepackage{cite}
\usepackage{caption}
\usepackage{subcaption}
\usepackage{multirow}
\usepackage{multicol}
\usepackage{latexsym}
\usepackage{amsmath,amssymb,enumerate}
\usepackage{xspace}
\usepackage{algorithm,algori thmicx}
\usepackage{float}
\usepackage{xcolor}
\usepackage{mathrsfs}
\usepackage{cleveref}
\usepackage{mathtools}
\usepackage{bm}
\usepackage{relsize}
\usepackage{thmtools}
\usepackage{thm-restate}
\usepackage[round]{natbib}

\usepackage{fullpage}
\usepackage{geometry}
\usepackage{setspace}

\allowdisplaybreaks

\newtheorem{theorem}{Theorem}[section]
\newtheorem*{theorem*}{Theorem}
\newtheorem{definition}{Definition}[section]

\newtheorem{claim}[theorem]{Claim}
\newtheorem{lemma}[theorem]{Lemma}

\newtheorem{corollary}[theorem]{Corollary}

\newcounter{note}[section]

\DeclareMathOperator*{\argmax}{arg\,max}

\def\sse{\subseteq}

\newcommand{\pr}{\mathbb{P}} 
\newcommand{\E}{\mathbb{E}}

\newcommand{\x}{\mathbf{x}}

\newcommand{\R}{\mathbb{R}}

\newcommand{\A}{\mathcal{A}}

\newcommand{\calH}{\mathcal{H}}
\newcommand{\calX}{\mathcal{X}}

\newcommand{\calY}{\mathcal{Y}}
\newcommand{\calP}{\mathcal{P}}

\newcommand{\olp}{\mathtt{OLPO}\xspace}
\newcommand{\olpLO}{\mathtt{Lin\text{-}OLPO}}
\newcommand{\olpop}{\emph{online linear-product optimization problem}\xspace}
\newcommand{\ignore}[1]{}

\newcommand{\btheta}{\bm{\theta}}
\newcommand{\bTheta}{\bm{\Theta}}

\newcommand{\blf}{\bm{f}}

\newcommand{\bg}{\bm{\gamma}}

\newcommand{\M}{M}
\newcommand{\w}{\mathbf{w}} 
\newcommand{\sprod}[2]{ \left\langle {#1},{#2} \right\rangle }

\newcommand{\PMTheta}{\{\pm1\}^M}
\newcommand{\ball}{\mathbb{B}}

\renewcommand{\cite}[1]{\citep{#1}}

\definecolor{mydarkblue}{rgb}{0,0.08,0.45}
\definecolor{mydarkred}{rgb}{0.4,0.08,0}
\hypersetup{
    colorlinks=true,
    citecolor=mydarkblue,
    linkcolor=red,
    }

\title{Improved and Oracle-Efficient Online $\ell_1$-Multicalibration}
\date{}
\author{Rohan Ghuge\footnote{H. Milton Stewart School of Industrial and Systems Engineering / Algorithms and Randomness Center, Georgia Institute of Technology, Atlanta, USA. Email: rghuge3@gatech.edu.} \and Vidya Muthukumar\footnote{School of Electrical and Computer Engineering/H. Milton Stewart School of Industrial and Systems Engineering, Georgia Institute of Technology, Atlanta, USA. Email: vmuthukumar8@gatech.edu. Supported in part by NSF awards IIS-2212182 and CCF-2239151 and gifts from Adobe and Amazon Research.
} \and Sahil Singla\footnote{School of Computer Science, Georgia Institute of Technology, Atlanta, GA, USA. Email: ssingla@gatech.edu. Supported in part by NSF awards CCF-2327010 and CCF-2440113.}}

\begin{document} 

\maketitle

\begin{abstract}
    We study \emph{online multicalibration}, a framework for ensuring calibrated predictions across multiple groups in adversarial settings, across $T$ rounds.
    Although online calibration is typically studied in the $\ell_1$ norm, 
    prior approaches to online multicalibration have taken the indirect approach of obtaining rates in other norms (such as $\ell_2$ and $\ell_{\infty}$) and then transferred these guarantees to $\ell_1$  at additional loss. In contrast, we propose a direct method that achieves improved  and oracle-efficient rates of $\widetilde{\mathcal{O}}(T^{-1/3})$ and $\widetilde{\mathcal{O}}(T^{-1/4})$ respectively, for online $\ell_1$-multicalibration.
    Our key insight is a novel reduction of online \(\ell_1\)-multicalibration to an online learning problem with product-based rewards, which we refer to as \emph{online linear-product optimization} ($\olp$). 

    To obtain the improved rate of $\widetilde{\mathcal{O}}(T^{-1/3})$, we introduce a linearization of $\olp$ and design a no-regret algorithm for this linearized problem. Although this method guarantees the desired sublinear rate (nearly matching the best rate for online calibration), it becomes computationally expensive when the group family \(\mathcal{H}\) is large or infinite, since it enumerates all possible groups. 
    To address scalability, we propose a second approach to $\olp$  that makes only a polynomial number of calls to an offline optimization (\emph{multicalibration evaluation}) oracle, resulting in \emph{oracle-efficient} online \(\ell_1\)-multicalibration with a  rate of $\widetilde{\mathcal{O}}(T^{-1/4})$. Our framework also extends to certain infinite families 
of groups (e.g., all linear functions on the context space) by 
exploiting a $1$-Lipschitz property of the \(\ell_1\)-multicalibration error with respect to \(\mathcal{H}\).
    
\end{abstract}







\section{Introduction}

Machine learning algorithms, powered by advances in data availability and model development, play a crucial role in decision-making across domains such as healthcare diagnostics, recidivism risk assessment, and loan approvals. 
This work focuses on forecasting algorithms that predict, say, the probability of binary outcomes 
$y$  (such as patient's severity or loan repayment) based on observable features $x$, in online settings where predictions are made as data is collected.
A key metric used to evaluate the performance of such probability forecasters is \emph{calibration}~\cite{dawid1982well}.
Roughly, it says that for any candidate prediction $p \in [0,1]$, the fraction of forecasts with prediction $p$ should converge to $p$. 
In 1998, the seminal work of   \citet{foster1998asymptotic}  showed a bound of $O(T^{-1/3})$ for online calibration in the $\ell_1$ metric. 
Since then, this $O(T^{-1/3})$  bound has been re-proved through insightful alternative approaches~\cite{hart2022calibrated,abernethy2011blackwell}. However, improving this bound is a challenging open problem that has only recently seen some progress~\cite{QiaoV-STOC21,dagan25breaking}.

Despite its popularity, calibration has a major limitation: calibrated predictions may perform poorly on specific sub-populations in the data, identifiable through contextual features such as gender, race, and age.
To address this issue, 
\citet{hebert2018multicalibration} proposed \emph{multicalibration}, a framework designed to address discrimination arising from data in the batch setting.
Informally, multicalibration is a requirement that the forecasts be statistically unbiased conditional both on its own prediction \emph{and} on membership in any one of a large collection of intersecting subsets of the data space $\calH$. 
Multicalibration and its variants have been an active area of research (see, e.g., \citet{kim2019multiaccuracy,jung2021moment,kim2022universal,GuptaJN+22,globus2023multical}). 
Multicalibration has also found applications in \emph{omniprediction}~\cite{gopalan2022op, gopalan2023swap}, a concept that asks for a single prediction which can be simultaneously used to optimize a large number of loss functions such that it is competitive with some hypothesis class $\calH$. 
Approximately multicalibrated models in the $\ell_1$ metric turn out to automatically be omnipredictors, in both the batch~\cite{gopalan2022op} and online~\cite{garg2024oracle} settings. 

We are especially focused on online multicalibration~\cite{GuptaJN+22, garg2024oracle}, which naturally generalizes the fundamental problem of sequential calibration. 
Existing approaches first provide multicalibration guarantees in the (weaker) $\ell_{\infty}$ or $\ell_2$ metric, and then transfer these guarantees to the $\ell_1$ metric at additional and possibly superfluous loss.
The resulting rates for online $\ell_1$-multicalibration are significantly weaker than those for online calibration.
A second drawback is that the runtime of most existing algorithms (with the exception of  \citet{garg2024oracle}, which we discuss later) for online multicalibration is typically linear in $|\calH|$~\cite{GuptaJN+22, lee2022online} .
These algorithms are hence inefficient in even simple practical scenarios (e.g. linear functions in $d$ dimensions) where the hypothesis class is usually exponential in relevant problem parameters.

In light of these considerations, the main motivating question of our work is the following:

\begin{quote}\emph{Is there an online $\ell_1$-multicalibration algorithm that guarantees $O(T^{-1/3})$ error? Can we design ``oracle-efficient'' algorithms for online $\ell_1$-multicalibration?}\end{quote}

In this work, we make progress towards answering both these questions. 
We propose a method that achieves improved
and oracle-efficient rates of $\widetilde{\mathcal{O}}(T^{-1/3})$ and $\widetilde{\mathcal{O}}(T^{-1/4})$, respectively, for online $\ell_1$-multicalibration.  
Our key insight is a novel reduction of online \(\ell_1\)-multicalibration to an online learning problem with product-based rewards, which we refer to as \emph{online linear-product optimization} ($\olp$). 
To obtain the improved rate of $\widetilde{\mathcal{O}}(T^{-1/3})$, we introduce a linearization of $\olp$ and design a no-regret algorithm for this linearized problem. Although this method guarantees the desired $\widetilde{\mathcal{O}}(T^{-1/3})$  rate, it becomes computationally expensive when the group family \(\mathcal{H}\) is large or infinite, since it enumerates all possible groups. 
To address scalability, we propose a second approach to $\olp$  that makes only a single call per round to an offline optimization (\emph{multicalibration evaluation}) oracle, resulting in \emph{oracle-efficient} online \(\ell_1\)-multicalibration with a  rate of $\widetilde{\mathcal{O}}(T^{-1/4})$. Our framework also extends to certain infinite families of groups (e.g., all linear functions on the context space) by exploiting a $1$-Lipschitz property of the \(\ell_1\)-multicalibration error with respect to \(\mathcal{H}\).
We discuss the basic setup of online multicalibration, our results and our techniques in the rest of this section. 

\subsection{Online Multicalibration} \label{sec:defn}
We now formally define the problem of online $\ell_1$-multicalibration~\cite{GuptaJN+22, garg2024oracle}. 
Let $\calX$ denote the context space and $\calY = [0, 1]$ denote the label domain, which we assume to be one-dimensional.
Let $\calH$ denote a collection of real-valued functions $h: \calX \to \R$.
We use 
$\calH_B = \big\{ h: \max_{\x \in \calX} |h(\x)| \leq B \big\}$
to denote the set of functions with maximum absolute value on the context space bounded by $B$, and make the mild assumption that $\calH \subseteq \calH_B$ throughout the paper. 
For $n \in \mathbb{N}$,  we use [n] to denote the set of integers $\{1, \ldots, n\}$. 
All of our algorithms will consider a discretized set of forecasts $\calP := \left\{0, \frac{1}{m}, \ldots, 1\right\}$ for some discretization parameter $m \in \mathbb{N}$. We use $M$ to denote the size of $\calP$, i.e., $M = |\calP| = m+1$.

Online prediction proceeds in rounds indexed by $t \in [T]$, for a given  horizon of length  $T$. 
In each round, the interaction between a learner and an adversary proceeds as follows:
\begin{enumerate}   \setlength{\itemsep}{0pt}
  \setlength{\parskip}{0pt}
  
    \item The adversary selects a context $\x_t \in \calX$ and a corresponding label $y_t \in \calY$.  
    
    \item The learner receives $\x_t$, but no information about $y_t$ is revealed.

    \item The learner selects a distribution $\w_t$ and outputs $p_t \in \calP$ sampled according to $\w_t$. 

    \item The learner observes $y_t$.
\end{enumerate}

The learner's interaction with the adversary results in a \emph{history} $\pi_T = \{\x_t, y_t, p_t\}_{t=1}^T$. 
We make no assumptions about the adversary; 
however, the learner is allowed to use randomness in making predictions. This 
induces a probability distribution
over transcripts, and
our goal is to design online algorithms that have low online $\ell_1$-multicalibration error
in expectation, which is defined as follows:

\begin{definition}[$\ell_1$-Multicalibration Error]
    
Given a transcript $\pi_T$, a function $h \in \calH$, we define the learner's online $\ell_1$-multicalibration error with respect to $h$ as 
\[ \textstyle 
    K(\pi_{T}, h) := \sum_{p \in \calP} \frac{1}{T} \Big|    \sum_{t \in S(\pi_{T}, p)} h(\x_t) \cdot (y_t - p_t)  \Big|,
\]
where we define $S(\pi_T, p) = \{t \in [T]: p_t = p\}$.  
Finally, we define  $\ell_1$-multicalibration error with respect to the family $\calH$ as 
$K(\pi_{T}, \calH) := \max_{h \in \calH} K(\pi_{T}, h)$.
\end{definition}
From here on, online multicalibration error will, by default, refer to online $\ell_1$-multicalibration error.
When clear from context, we will drop $\pi_T$ and use $K(h)$ or $K(\calH)$ to denote the learner's online multicalibration error.

\subsection{Our Results}\label{sec:results}
Our first main result establishes an $\widetilde{\mathcal{O}}(T^{-1/3})$ rate for online $\ell_1$-multicalibration when the hypothesis class $\calH$ is finite. 

\begin{theorem}\label{thm:main-inf}
    There is an algorithm that achieves online $\ell_1$-multicalibration error with respect to $\calH$ with 
    \[\textstyle
     \E[K(\pi_T, \calH)] \,\, \leq \,\ O\left(BT^{-1/3}\sqrt{\log(6T|\calH|)}\right) \enspace .
    \]
    The running time of this algorithm is linear in $|\calH|$ and polynomial in $T$.
\end{theorem}

This improves over an $\widetilde{O}( T^{-1/4} )$ bound obtained in \citet{GuptaJN+22} via online $\ell_{\infty}$-multicalibration (without needing to go through ``bucketed'' predictions), and nearly matches the best known bound for online calibration \cite{dagan25breaking}.

\paragraph{Comparison to \citet{noarov2025high}.} After the initial submission of this work, it has been brought to our attention that a similar bound for online $\ell_1$-multicalibration for finite hypothesis class $\calH$ could be derived from Theorem 3.4 of \citet{noarov2025high}.
We note that their algorithmic framework is quite different from ours, despite both works using an expert routine. 
In particular, their framework requires a small-loss regret bound to get the result, while a worst-case regret bound suffices in ours. 
Additionally, our algorithm is simpler (e.g. not requiring the solution of any min-max optimization problem) and proceeds through a novel reduction to the $\olp$ problem (defined in Section~\ref{sec:olp}). We believe that this reduction is of independent interest as it facilitates the oracle-efficient results in a more natural and modular way.

Although  \Cref{thm:main-inf} obtains an $\widetilde{\mathcal{O}}(T^{-1/3})$ rate for online $\ell_1$-multicalibration, it does not apply to infinite-sized hypothesis class $\calH$. To address this, our next result obtains bounds in terms of the ``covering number'' of $\calH$.

\begin{definition}[$\beta$-cover in $L_{\infty}$ metric,~\citet{bronshtein1976varepsilon}]\label{def:covering}
For any function class $\calH$, a finite subset of functions $\calH_{\beta} = \{h^{1},\dots,h^{N}\} \subseteq \calH$ is a $\beta$-cover with respect to the $L_{\infty}$ metric if for every $h \in \calH$, there exists some $i \in [N]$ such that $\max_{\bold{x} \in \calX} 
 |h(x) - h^i(x)|  \leq \beta$.
\end{definition}

Let $\calH_{\beta}$ denote a smallest possible $\beta$-cover for $\calH$. We show that the online multicalibration error for infinite-sized hypothesis class $\calH$ can be bounded in terms of  $|\calH_{\beta}|$.

\begin{theorem}\label{thm:main-inf-infinite}
    There is an algorithm that achieves online $\ell_1$-multicalibration error with respect to $\calH$ with
    \[ \textstyle
     \E[K(\pi_T, \calH)] \,\, \leq \,\  O\left(BT^{-1/3}\sqrt{\log(6T|\calH_{\beta}|)}\right) + \beta.
    \]
    The running time of this algorithm is linear in $|H_{\beta}|$ and polynomial in $T$.
\end{theorem}

As a consequence of this result, we obtain immediate applications to polynomial regression and bounded, Lipschitz convex functions (see \Cref{sec:olp-infinite} for details). We state one corollary here for the class of bounded linear functions.

\begin{corollary}\label{cor:linear-f}
Suppose $\calX = [0, 1]^d$ and  $\calH = \{h \in \R^d: \| h \|_1 \leq B \}$ is the class of bounded linear functions.  Then, there is an algorithm with runtime $O((B\sqrt{T})^d)$ that achieves online $\ell_1$-multicalibration error with respect to $\calH$ with
\[\textstyle
\E[K(\pi_T, \calH)] = \mathcal{O}\left(Bd^{1/2}T^{-1/3}\log(BT)\right). 
\]
\end{corollary}

Next, we give an ``oracle-efficient'' algorithm for online $\ell_1$-multicalibration for large hypothesis class $\calH$. Our earlier algorithms need to enumerate over $\calH$ or its $\beta$-cover, both of which are often exponentially large.
We show how this can be avoided using the following  \emph{offline} oracle.

\begin{definition}[Offline Oracle]
 We receive a sequence of contexts $\{\x_s\}_{s = 1}^t$ with corresponding reward vectors $\{\blf_s\}_{s = 1}^t$, coefficients $\{\kappa_s\}_{s=1}^t$, and an error parameter $\epsilon > 0$.     The offline oracle returns a solution $({h}^*, {\btheta}^*) \in \calH \times \PMTheta$ such that
\[
    \sum_{s=1}^t  \kappa_s \sprod{{\btheta}^*}{{h}^*(\x_s)\cdot \blf_s} \geq \max_{h \in \calH, \btheta \in \PMTheta} \Big\{\sum_{s=1}^{t}  \kappa_s \sprod{\btheta}{h(\x_s) \cdot\blf_s}\Big\} - \epsilon \enspace .
\]
\end{definition}

Given such an oracle, our main result is the following.

\begin{theorem}\label{thm:main-oracle}
There is an algorithm that achieves oracle-efficient online $\ell_1$-multicalibration with respect to a binary-valued $\calH: \calX \to \{0,1\}$ with 
\begin{equation*}
\textstyle
\E[K(\pi_T, \calH)] \leq \widetilde{O}\big(T^{-1/4} \sqrt{\log (T)}\big),
\end{equation*}
under the assumptions of either \emph{transductive} or \emph{sufficiently separated} contexts~\cite{syrgkanis2016efficient} (see~\Cref{sec:oracle-apps} for formal definitions).
Moreover, this algorithm requires only a single call to  the offline oracle per round.
\end{theorem}
This result improves over the $\widetilde{O}( T^{-1/8} )$ bound obtained in \citet{garg2024oracle},
who also provide an oracle-efficient multicalibration algorithm but require access to an \emph{online regression oracle}.
Our oracle is equivalent to evaluating the online $\ell_1$-multicalibration error for a sequence of prediction in a one-shot manner --- essentially, an \emph{offline oracle}.
Offline oracles are considered to be easier than online oracles.
The assumptions in Theorem~\ref{thm:main-oracle} on contexts and of binary-valued $\calH$ are required to make the oracle \emph{implementable} while maintaining the ``stability" of the online algorithm, and are also commonly used in oracle-efficient online learning~\cite{syrgkanis2016efficient,dudik2020oracle}.
We also provide a generic ``black-box'' bound  in \Cref{thm:oracle-efficient-full}.

\subsection{Our Techniques}\label{sec:techniques}
At the heart of our approach is a  novel reduction of online $\ell_1$-multicalibration to an online learning problem with product-based rewards, which we refer to as \emph{online linear-product optimization} ($\olp$). In particular, we show that any learning algorithm for $\olp$ with  regret $R_T(\mathcal{L}; \calH)$ 
can be efficiently transformed into an algorithm with online $\ell_1$-multicalibration error $R_T(\mathcal{L}; \calH)$, up to a small error (see \Cref{thm:reduction}). This reduction is crucial in both our improved and oracle-efficient rates for online $\ell_1$-multicalibration. 

\medskip
\noindent {\bf Improved Rates.}
Our first set of results establish the best-known information-theoretic rates for online $\ell_1$-multicalibration, beginning with the case of a finite-sized hypothesis class. We achieve this by designing a no-regret algorithm for $\olp$.
The main challenge in solving $\olp$ is that  it is unclear \emph{apriori} how to perform online optimization on a reward function that involves a product of variables. 
To address this, we define an online linear optimization problem in a higher-dimensional space and show that 
$\olp$ reduces to this problem, thereby effectively \emph{linearizing} the reward structure—at the cost of enumerating over all $h \in \calH$. 
We denote this new problem as $\olpLO$ (or \emph{linearized} \olpop).
The reduction from $\olp$ to $\olpLO$ introduces two issues.
The first issue is that we need to ensure that the best fixed actions in $\olp$ and $\olpLO$ remain consistent. To address this, we introduce a specific mixed norm and restrict the set of actions of $\olpLO$ to a unit ball in this mixed norm.  
The second issue is that the expanded decision space may  introduce better actions for $\olpLO$ that are not valid in $\olp$. 
To resolve this, we  appropriately scale decisions as we translate them from $\olpLO$ to $\olp$, ensuring consistency between the action spaces of $\olp$ and $\olpLO$ (see \Cref{clm:olp-reduction}).

Next, we design a no-regret algorithm for $\olpLO$ with the mixed-norm constraint that carefully combines two components: (i) a reward-maximizing no-regret online linear optimization (OLO) algorithm (e.g., online gradient descent), and (ii) a reward-maximizing no-regret algorithm for the experts setting (e.g., multiplicative weights update). 
Our approach runs $|\calH|$ instances of the OLO algorithm in parallel. Each OLO instance executes the action optimal for calibration with respect to a specific hypothesis $h \in \calH$. Meanwhile, the experts algorithm identifies hypotheses that appear ``more difficult" with respect to calibration error—where difficulty is measured using the cumulative reward of the corresponding OLO algorithm.
We obtain the final guarantee by analyzing the regret of this algorithm (see \Cref{lem:no-regret-olp-lin}) and combining it with the two earlier reductions.

To transfer these bounds to an infinite-sized hypothesis class $\calH$ (\Cref{thm:main-inf-infinite}), we  first show that  online multicalibration error is $1$-Lipschitz with respect to $\calH$. Then, we construct an appropriate \emph{covering} of the hypothesis class $\calH$, and appeal to online $\ell_1$-multicalibration rates for finite $\calH$.

\medskip

\noindent {\bf Oracle Efficiency.}
A plethora of oracle-efficient online learning algorithms have been developed in the last decade, with the aim of efficient online regret minimization  given an \emph{offline} optimization oracle \cite{syrgkanis2016efficient,daskalakis2016learning,dudik2020oracle}. 
We cannot easily adapt these frameworks to $\olpLO$, primarily because the algorithm needs to maintain and access $|\calH|$ parallel copies of OLO algorithms, even if Step~(ii) above is made efficient.
Fundamentally, $\olpLO$ operates in an augmented space that is linear in $|\calH|$ and therefore intractable.

We instead adapt the oracle-efficient framework directly to the $\olp$ problem, circumventing the need to access the augmented $\olpLO$ structure.
We leverage the \emph{generalized Follow-the-Perturbed-Leader} family of algorithms~\cite{dudik2020oracle} and show that, remarkably, a regret bound for $\olp$ can be derived using similar techniques as in~\citet{dudik2020oracle} using certain special properties of the $\olp$ structure.
Specifically, it suffices to restrict decisions for $\olp$ to the Boolean hypercube; i.e., $\PMTheta$. This allows us to ultimately prove \Cref{thm:main-oracle}.

\section{Further Related Work}\label{sec:related-work}

\noindent {\bf Calibration.}  
The notion of (sequential) calibration has been extensively studied in statistical learning and forecasting~\cite{dawid1982well, foster1998asymptotic, hart2022calibrated}. 
\citet{dawid1982well} introduced the notion of calibration, and \citet{foster1998asymptotic} showed the existence of an algorithm capable of producing calibrated forecasts in an online adversarial setting.  
Subsequently, numerous algorithms were discovered for calibration; see, for example, \citet{fudenberg1999easier, hart2000simple, sandroni2003calibration, sandroni2003reproducible, foster2006calibration, perchet2009calibration}. 
\citet{foster1999proof} has given a calibration algorithm based on Blackwell approachability, while \citet{abernethy2011blackwell} showed a connection between calibration and no-regret learning (via Blackwell approachability).

\medskip
\noindent {\bf Multicalibration.} The concept of multicalibration extends standard calibration by requiring calibrated predictions not just overall, but across multiple subpopulations defined by a hypothesis class. This notion was first introduced by \citet{hebert2018multicalibration} in the batch setting. They showed that multicalibrated predictors could provide strong fairness guarantees while maintaining predictive accuracy.
Since then, multicalibration and some analogous notions have been studied; for example, \citet{kim2019multiaccuracy} study multiaccuracy, \citet{jung2021moment} give algorithms for moment multicalibration, and \citet{GuptaJN+22} investigates quantile multicalibration.
Another line of work explores the connection between multicalibration and \emph{omniprediction}~\cite{gopalan2022op, gopalan2023swap, garg2024oracle}.
Omniprediction is a paradigm for loss minimization that was introduced in \citet{gopalan2022op}. Informally, an omnipredictor is a prediction algorithm that could be used for minimizing a large class of loss functions such that its performance is comparable to some benchmark class of models $\mathcal{F}$. 
\citet{gopalan2022op} show that we can reduce omniprediction to a $\ell_1$-multicalibration. In particular, if a prediction algorithm is \emph{multicalibrated} (in the $\ell_1$ metric) with respect to some benchmark class of models $\calH$, then it is an omnipredictor with respect to all Lipschitz convex losses and the class $\calH$. The problem of online omniprediction was introduced in \citet{garg2024oracle}, and they used similar ideas to that of \citet{gopalan2022op} and \citet{globus2023multical} to show that online omniprediction can be reduced to online $\ell_1$-multicalibration. Furthermore, they provided an efficient reduction from online multicalibration to online squared error regression over $\calH$, yielding oracle-efficient algorithms for online $\ell_1$-multicalibration and, consequently, for online omniprediction.
Other works have also explored online multicalibration under different settings; for example, \citet{GuptaJN+22} and \citet{lee2022online} provided algorithms that guarantee online multicalibration in the $\ell_{\infty}$ metric. 

\medskip
\noindent {\bf Oracle Efficient Algorithms and Online Multigroup Learning.} 
No-regret algorithms based on injecting random perturbations, such as Follow-the-Perturbed-Leader, have an early history in lending themselves to computational efficiency on specially structured combinatorial problems, such as the shortest path problem and prediction with decision trees~\cite{kalai2005efficient}. In these problems, the ``oracle" constitutes solving a shortest path problem or learning an optimal decison tree from offline data.
Motivated by this,~\citet{hazan2016computational} posed the more general problem of achieving computational efficiency in \emph{online} learning with respect to an \emph{offline} optimization oracle and showed that this goal is not achievable in the worst case.
Subsequently,~\citet{daskalakis2016learning,syrgkanis2016efficient,dudik2020oracle,wang2022adaptive,haghtalab2022oracle,block2022smoothed} showed that oracle-efficient learning is possible under further assumptions---both for a variety of combinatorial settings involving market design, and for online supervised learning involving contexts $\{\x_t\}_{t=1}^T$ and labels $\{y_t\}_{t=1}^T$.
For the learning settings the results make assumptions on the contexts $\{\x_t\}_{t=1}^T$, but not on the labels.
The setting of online $\ell_1$-multicalibration is more reminiscent of (but not exactly the same as) the latter case of online supervised learning.
Accordingly, we adopt the assumptions of \emph{transductive} or \emph{sufficiently separated data} made in~\cite{dudik2020oracle} (which subsume those made in~\cite{daskalakis2016learning,syrgkanis2016efficient}) and believe our results could be adapted to the weaker assumption of \emph{smoothed data} or \emph{K-hint data}~\cite{haghtalab2022oracle,block2022smoothed} in future work.
Recently, the oracle-efficient framework was also adopted for online multigroup learning with the aim of minimizing \emph{group-regret}~\cite{acharya2023oracle,deng2024group}.
While group-regret could be closely related to multicalibration and the associated $\olp$ instance that we set up, it is unclear how to adapt the techniques in~\cite{deng2024group}, which are tailored to binary labels and loss functions, to the $\olp$ problem. 

There also exists a rich body of work on oracle efficiency with respect to either \emph{online regression oracles} or \emph{cost-sensitive classification oracles} in the contextual bandits literature (see, e.g.~\cite{agarwal2014taming,foster2020beyond}).
Indeed, an online regression oracle was assumed by~\cite{garg2024oracle} but in a completely different manner from the contextual bandit application.
We do not adopt these frameworks for oracles because the contextual bandit problem is different in scope and unnecessary to solve for multicalibration.
It is also arguably more difficult than full-information contextual learning: the strongest of the aforementioned results assume stochasticity in the labels $\{y_t\}_{t=1}^T$ and that they are \emph{realized} by a function in the hypothesis class, and still require an online regression oracle, which is stronger than an offline oracle and is only known to be solvable in the special case of linear models~\cite{azoury2001relative}.



\section{Reducing Multicalibration to $\olp$ }\label{sec:reduction}
In this section, we show how to efficiently reduce  online multicalibration to an online learning problem
with product-based rewards, which we refer to as \emph{online linear-product optimization} ($\olp$), and define next. This reduction will be crucial to our improved rates for online multicalibration.
We note that there is precedent for connections between calibration and regret; in particular,~\cite{abernethy2011blackwell} provided a simpler reduction between calibration and online linear optimization.

\subsection{Online Linear-Product Optimization} \label{sec:olp}
We formally define the \olpop ($\olp$). 
Let $\calX$ denote the context space and let hypothesis class $\calH \sse \calH_B$ be a collection of $B$-bounded real-valued functions $h: \calX \to \R$. 
Let $\ball_{\infty} \sse \R^{\M}$ denote the unit cube, i.e., $\ball_{\infty} = \{\btheta \in \R^{\M}: \|\btheta\|_{\infty} \leq 1\}$.
The set $\calH \times \ball_{\infty}$ will denote an action set. 
In each round $t \in [T]$:
\begin{enumerate}\setlength{\itemsep}{0pt}
  \setlength{\parskip}{0pt}
  \vspace{-0.1cm}
    
    \item Learner plays a function $h_t \in \calH$ and vector $\btheta_t \in \ball_{\infty}$.
    
    \item Adversary then reveals a context $\x_t$ and a reward vector $\blf_t \in \R^M$.
    
    \item Learner then receives reward $ \sprod{\btheta_t}{h_t(\x)\cdot \blf_t}$.  \\
\end{enumerate}

Note that this is not a standard online linear optimization problem since it involves a product of variables.
We will use $\mathcal{L}$ to denote a generic algorithm for $\olp$. This algorithm takes as input a sequence of vectors $(\x_1,\blf_1),\ldots,(\x_{t-1},\blf_{t-1})$ and returns a pair $(h_t, \btheta_t) \in \calH \times \ball_{\infty}$.
We denote by $R_{T}(\mathcal{L}; (\x_1, \blf_1) \ldots, (\x_{T}, \blf_{T}); \calH)$ the regret of $\mathcal{L}$ when compared to the best fixed action $(h^*, \btheta^*)$: 
\begin{equation*}
R_{T}(\mathcal{L}; (\x_1, \blf_1) \ldots, (\x_{T}, \blf_{T}); \calH) := \max_{h^* \in \calH, \btheta^* \in \ball_{\infty}} \Big\{ \sum_{t=1}^T h^*(\x_t) \cdot \sprod{\btheta^*}{\blf_t} \Big\} - \sum_{t=1}^T h_t(\x_t)\cdot \sprod{\btheta_t}{\blf_t}.
\end{equation*}
When the input sequence is clear, we will omit it from the definition and simply write $R_{T}(\mathcal{L}; \calH)$.

Recall that in online multicalibration, the learner receives a context $\x_t$ in each round $t \in [T]$ and makes a prediction $p_t \in \calP= [1/m]$ according to some distribution $\w_t$.
Then, $\w_t$ is a $\M$-dimensional probability vector (recall that $M = |\calP|$), where $\w_t(i)$ represents the probability that $p_t$ equals $i/m$. 
The main result of this section shows that an  algorithm for $\olp$ can be efficiently converted into an algorithm for online $\ell_1$-multicalibration.

\begin{theorem}\label{thm:reduction}
Let $\mathcal{L}$ be an  algorithm for $\olp$ for some collection $\calH \sse \calH_B$ with expected regret denoted by $R_{T}(\mathcal{L}; \calH)$. 
Then, there is an algorithm such that its sequence of predictions has online $\ell_1$-multicalibration error with respect to $\calH$ bounded as
\begin{equation}\label{eq:l1-reduction}
\E[K(\pi_T, \calH)] \leq \frac{B}{m} + \frac{R_{T}(\mathcal{L}; \calH)}{T} + 4B\sqrt{\frac{m\log(6T|\calH|)}{T}} + \frac{4mB}{T}. 
\end{equation}  
Moreover, this algorithm is efficient; its running time is polynomial in the running time of $\mathcal{L}$, the  discretization parameter $m$, and $T$. 
\end{theorem}

We will instantiate \Cref{thm:reduction} with different online learning algorithms for $\olp$ in \Cref{sec:l1-guarantees} and \Cref{sec:oracle-olp}. 
In the remainder of this section, we outline the reduction and proof sketch of~\Cref{thm:reduction}. 
In this reduction, the context space $\calX$ and the collection of functions $\calH$ remain unchanged. 
The key steps will involve utilizing the fact that the actions are in $\calH \times \ball_{\infty}$ and carefully setting the reward vectors $\blf_1, \blf_2, \ldots, \blf_T \in \R^{\M}$.

\subsection{The Reduction}
In order to convert an  algorithm for $\olp$ into an online multicalibration algorithm, we need a ``halfspace oracle''.
One such oracle was used to reduce calibration to Blackwell approachability (and subsequently to no-regret learning) in \cite{abernethy2011blackwell}.

\begin{definition} [Halfspace Oracle] \label{def:halspace}
    We assume access to an efficient halfspace oracle $\cal{O}$ that can, given $\x \in \calX$, $h \in \calH$ and $\btheta \in \ball_{\infty}$, select a  probability distribution $\w \in \R^{\M}$ with $\|\w\|_1=1$ such that
    for all $y \in \calY$, we have 
    \begin{equation*}
    \textstyle
    \sum_{i=0}^m \btheta(i) \cdot h(\x) \cdot \w(i) \cdot \left(y - \frac{i}{m}\right) \,\, \leq \,\, \frac{B}{m}.
    \end{equation*}
\end{definition}

A surprising result of \cite{abernethy2011blackwell} shows that an efficient halfspace oracle always exists in the context of calibration for 
$\calY=[0,1]$. 
We show that this result extends to our setting, i.e., given $h$ and $\x$, the oracle construction remains unchanged.

\begin{lemma}[\Cref{alg:l1-oracle}]\label{lem:halfspace-oracle}
Given any $\x \in \calX$, $h \in \calH$, and $\btheta \in \ball_{\infty}$, there exists an efficient halfspace oracle. 
\end{lemma}

For completeness, we formalize this lemma and present its proof in \Cref{sec:reduction-halfspace}.
We are now ready to describe our multicalibration algorithm. 

\paragraph{Multicalibration Algorithm.}
At each round $t \in [T]$, the algorithm
randomly predicts  $p_t$ according
to some distribution $\w_t \in \R^{\M}$.
The distribution $\w_t$ is obtained using a combination of the output of the learning algorithm $\cal{L}$ for $\olp$ and the halfspace oracle $\cal{O}$. In particular, given previous contexts $\x_1, \ldots, \x_{t-1}$ and  previous vectors $\blf_{1}, \ldots, \blf_{t-1}$, 
let $(h_t, \btheta_t)$ denote the action selected by $\cal{L}$ in round $t$. The prediction distribution for current context $\x_t$ is now obtained using the halfspace oracle: $\w_t = {\cal O}(\x_t, h_t, \btheta_{t})$. Finally, after observing $y_t$, define $\blf_t := \blf_t(\w_t, y_t)$ where 
\begin{align}\label{eq:lt-def}
 \blf_t(\w_t, y_t) = \left( \w_t(0) \cdot \left(y_t - 0\right),  \cdots, \w_t(i) \cdot \left(y_t - \frac{i}{m}\right) \cdots \right).
\end{align}
We formally describe the reduction in \Cref{alg:reduction}.

\begin{algorithm}[H]
\caption{\textsc{ Online $\ell_1$-Multicalibration}}
\label{alg:reduction}
\begin{algorithmic}[1]
\For{$t = 1, \ldots T$}
\State observe $\x_t$. 
\State query the $\olp$ algorithm: $(h_{t}, \btheta_{t}) \gets \mathcal{L}((\x_{1}, \blf_1), \ldots, (\x_{t-1}, \blf_{t-1}))$.
\State query the halfspace oracle: $\w_{t} \gets \mathcal{O}(\x_t, h_t,  \btheta_{t})$.
\State predict $p_t$ according to the distribution $\w_t$ and observe $y_t$.
\State ${\blf}_t \gets \blf_t(\w_t, y_t) = \left( \w_t(0) \cdot \left(y_t - 0\right),  \cdots, \w_t(i) \cdot \left(y_t - \frac{i}{m}\right) \cdots \right)$
\EndFor
\end{algorithmic}
\end{algorithm}

\subsection{Reduction Analysis}
We now analyze \Cref{alg:reduction} and complete the proof of~\Cref{thm:reduction}.
At a high level, the proof of correctness works by first bounding the expected multicalibration error for any group $h$ by $\|\frac{1}{T} \sum_{t=1}^T h(\x_t)\cdot \blf_t\|_1$ (see \eqref{eq:exp-l1-2}). 
Then, using the definition of the dual norm, we  write 
\[ 
\left\|\frac{1}{T} \sum_{t=1}^T h(\x_t)\cdot \blf_t\right\|_1 = \frac{1}{T} \sup_{\|\btheta\|_{\infty} \leq 1} \sprod{\btheta}{\sum_{t=1}^T h(\x_t) \cdot \blf_t}.
\]
Thus, the overall multicalibration error can be (roughly) related to the reward of the corresponding instance of $\olp$. 
We note that we can replace the ``sup'' operator above with a ``max'' operator by the compactness of $\ball_{\infty} = \{\btheta: \|\btheta\|_{\infty} \leq 1\}$ and the continuity of the linear function. 

\begin{proof}[Proof of \Cref{thm:reduction}]
We first  reduce  online $\ell_1$-calibration error for any group $h$ to the $\ell_1$-norm of the vector $\frac{1}{T} \sum_{t=1}^T h(\x_t)\cdot \blf_t$ plus a small additive ``error'' term.
Recall, the  expected online $\ell_1$-multicalibration for group $h \in \calH$: 
\begin{align}
     \E\left[K(\pi_{T}, h)\right] 
     &= \,\, \frac{1}{T} \,\, \E\left[\sum_{p \in \calP}  \left|  \left( \sum_{t=1}^T \mathbb{I}\{p_t = p\}\cdot  h(\x_t) \cdot (y_t - p)\right) \right|\right]. \label{eq:exp-l1-1} 
\end{align}
We use the following lemma to relate the indicator random variables $\{\mathbb{I}\{p_t = p\}\}_{p \in \calP}$ to their corresponding expectations $\{\w_t(p)\}_{p \in \calP}$, where $\w_t(p)$ denotes the probability that $p_t$ equals $p$.

\begin{restatable}{lemma}{goodevent}\label{lem:good-event}
    We have 
    \[
        \E \left[\max_{h \in \calH} \left\{ \sum_{p \in \calP}\left| \sum_{t=1}^T \Big(\mathbb{I}\{p_t = p\} - 
        \w_t(p) \Big)\cdot  h(\x_t) \cdot (y_t - p) \right| \right\} \right] \leq 4B\sqrt{Tm\log(6T |\calH|)}  \,\, + \,\, 4mB.
    \]
\end{restatable}

The proof of this lemma relies on a vector version of Azuma–Hoeffding inequality and is deferred to \Cref{sec:lemmaVecAzuma}. We will next use it to complete the proof of  \Cref{thm:reduction}.

We now proceed to upper bound the online multicalibration error for $\calH$. Using \eqref{eq:exp-l1-1} along with the triangle inequality, we get 
\begin{align}
\E\left[K(\pi_{T}, \calH)\right] \,\, &=  \,\, \frac{1}{T} \,\,\E\left[ \max_{h \in \calH}\left\{ \sum_{p \in \calP}  \left|   \sum_{t=1}^T \mathbb{I}\{p_t = p\}\cdot  h(\x_t) \cdot (y_t - p)\right|\right\}\right] \notag \\
&\leq \,\, \frac{1}{T} \,\,\max_{h \in \calH}\left\{\sum_{p \in \calP}  \left|   \sum_{t=1}^T \w_t(p)\cdot  h(\x_t) \cdot (y_t - p) \right|\right\} +  \notag \\
& \qquad \qquad  \frac{1}{T} \,\,\E \left[\max_{h \in \calH} \left\{ \sum_{p \in \calP}\left| \sum_{t=1}^T (\mathbb{I}\{p_t = p\} - 
        \w_t(p))\cdot  h(\x_t) \cdot (y_t - p) \right| \right\} \right] \notag \\
&\leq \,\, \frac{1}{T} \,\, \max_{h \in \calH}\left\{\left\|\sum_{t=1}^T h(\x_t) \cdot \blf_t\right\|_1\right\} + 4B\sqrt{\frac{m\log(6T|\calH|)}{T}} + \frac{4mB}{T}\label{eq:exp-l1-2},
\end{align}
where the last inequality follows from \Cref{lem:good-event} together with the definition of $\blf_t$ in~\eqref{eq:lt-def}.
Now using the definition of the dual norm, we can write 
\begin{align}
   \frac{1}{T}\,\,\max_{h \in \calH}\left\| \sum_{t=1}^T h(\x_t) \cdot \blf_t\right\|_1 \,\, &= \,\, \frac{1}{T} \,\, \max_{h \in \calH, \btheta \in \ball_{\infty}} \sprod{\btheta}{\sum_{t=1}^T h(\x_t) \cdot \blf_t} \notag \\
    &\leq \,\, \frac{1}{T} \,\, \sum_{t=1}^T \sprod{\btheta_{t}}{h_t(\x_t)\cdot\blf_t} +  \frac{R_{T}(\mathcal{L}; \calH)}{T} \,\, \leq \,\, \frac{B}{m} + \frac{R_{T}(\mathcal{L}; \calH)}{T}, \label{eq:l1-olp-regret}
    \end{align}
where the first inequality follows by applying the regret bound for $\olp$ obtained by  $\mathcal{L}$ and the second inequality follows from the definition of the halfspace oracle (see \Cref{def:halspace}).
Combining \eqref{eq:exp-l1-2} and \eqref{eq:l1-olp-regret} completes the proof of \Cref{thm:reduction}.
\end{proof}

\section{Improved  Online Multicalibration}\label{sec:l1-guarantees}
In this section, we give upper bounds for online multicalibration foregoing computational complexity considerations. 
In \Cref{sec:multical-finite}, we consider the case of a finite hypothesis class $\calH$ and provide upper bounds through the design of an appropriate no-regret algorithm for $\olp$. Combined with \Cref{thm:reduction} and \Cref{clm:olp-reduction}, this yields sublinear bounds for online multicalibration when the hypothesis class $\calH$ is finite, proving \Cref{thm:main-inf}.
In \Cref{sec:olp-infinite}, we describe the changes necessary to handle an infinite-sized hypothesis class $\calH$ and prove \Cref{thm:main-inf-infinite}.

\subsection{Online Multicalibration for Finite Groups}\label{sec:multical-finite}
We describe a no-regret algorithm for $\olp$ in the case where the hypothesis class $\calH$ is finite. 
The key idea is to introduce an online linear optimization problem in a higher-dimensional space, and to subsequently reinterpret $\olp$ in terms of this problem. 
We denote this problem as  $\olpLO$ (or \emph{linearized} \olpop). 
We set up some preliminaries to define $\olpLO$.
We index the elements of $\calH$ by $h^{(1)}$, $h^{(2)}$, $\ldots$ $h^{|\calH|}$.
Given a vector $\bold{z} \in \R^{|\calH| \times M}$, let $\bold{z}(h)$ denote the $M$-dimensional sub-vector of $\bold{z}$ indexed by $h \in \calH$. Then, we define the  mixed norms:
\[
\|\bold{z}\|_{1,\infty} := \sum_{h \in \calH} \|\bold{z}(h)\|_{\infty} \qquad\text{ and }\qquad \|\bold{z}\|_{\infty, 1} := \max_{h \in \calH} \|\bold{z}(h)\|_{1}
\]
and use $\ball_{1, \infty} = \{\widetilde{\btheta}: \|\widetilde{\btheta} \|_{1, \infty} \leq 1\}$ and $\ball_{\infty, 1} = \{\widetilde{\blf}: \|\widetilde{\blf} \|_{\infty, 1} \leq 1\}$ to denote a unit in the respective norms.

We now define $\olpLO$, which is an online linear optimization problem over $\ball_{1, \infty}$.
\begin{definition}[$\olpLO$]
In round $t \in [T]$ of $\olpLO$, the learner plays (a possibly random) action $\widetilde{\btheta}_t \in \ball_{1, \infty}$. 
The adversary reveals a reward vector $\widetilde{\blf}_t \in \ball_{\infty, 1}$, 
which leads to a \emph{linear} reward $\sprod{\widetilde{\btheta}_t}{\widetilde{\blf}_t}$.
The goal of an online learning algorithm for $\olpLO$, denoted henceforth by $\widetilde{\mathcal{L}}$, is to minimize the regret, which is defined as follows:
\begin{align*}
R^{\olpLO}_T(\widetilde{\mathcal{L}}; {\ball}_{1,\infty}) := \max_{\widetilde{\btheta} \in {\ball}_{1, \infty}} \sprod{ \widetilde{\btheta}}{\sum_{t=1}^T \widetilde{\blf}_t} - \sum_{t=1}^T \E\left[\sprod{\widetilde{\btheta}_t}{ \widetilde{\blf}_t}\right],
\end{align*}
where the expectation is taken over possible randomness in $\widetilde{\btheta}_t$. 
\end{definition}

We now present a reduction from $\olp$ to $\olpLO$.
Recall that the primary challenge with $\olp$ is that the reward function involves a product of variables.
However, by appropriately ``expanding'' the reward vectors $\blf_1, \blf_2, \ldots,$ and the decision space to cover all possible $h \in \calH$, we can effectively \emph{linearize} the reward function, aligning it with the structure of $\olpLO$. 
This linearization introduces two key challenges: (1) the need to 
ensure that the best fixed actions in $\olp$ and $\olpLO$ remain consistent, and (2) preventing the expanded decision space in 
$\olpLO$ from introducing ``better'' actions that cannot be captured by $\olp$ (see \Cref{sec:techniques} for a more detailed discussion).

\begin{lemma}\label{clm:olp-reduction}
    Let $\widetilde{\mathcal{L}}$ be an online learning algorithm for 
    $\olpLO$ with corresponding expected regret $R_T^{\olpLO}(\widetilde{\mathcal{L}}; {\ball}_{1, \infty})$. Then, there exists a randomized learning algorithm $\mathcal{L}$ for $\olp$ with the decision set $\calH_B \times \ball_{\infty}$ such that its expected regret $R_T(\mathcal{L}; \calH_B) = B \cdot L \cdot R_T^{\olpLO}(\widetilde{\mathcal{L}}; \widetilde{\ball}_{1, \infty})$, where $L = \max_{t \in [T]}\{\|\blf_t\|_1\}$ is the maximum $\ell_1$-norm of the reward vectors in the $\olp$ instance.
\end{lemma}

\begin{proof}
We begin by describing how to obtain a learning algorithm $\mathcal{L}$ for $\olp$ using a learning algorithm $\widetilde{\mathcal{L}}$ for $\olpLO$. 
To achieve this, two components are required: (1) translating decisions from $\widetilde{\mathcal{L}}$ to $\mathcal{L}$, and (2) generating reward vectors for the $\olpLO$ instance using rewards from the $\olp$ instance.
For each round $t \in [T]$, given decision $\widetilde{\btheta}_t$ from $\widetilde{\mathcal{L}}$, set $h_t = h$ with probability $\gamma_t(h) = \|\widetilde{\btheta}_t(h)\|_{\infty}$, and $\btheta_t = (1/\gamma_t(h))\cdot\widetilde{\btheta}_t(h_t) \in \ball_{\infty}$.
Then, $(h_t, \btheta_t)$ corresponds to the decision taken by $\mathcal{L}$ for the $\olp$ instance. 
Next, given context $\x_t$ and reward vector $\blf_t$ from the $\olp$ instance, let
$\widetilde{\blf}_t = \frac{1}{B\cdot L}\cdot(h^{(1)}(\x_t)\cdot \blf_t, \ldots, h^{(|\calH|)}\cdot \blf_t )$ be the reward vector used to simulate the $\olpLO$ instance.
Note that $\widetilde{\blf}_t \in \ball_{ \infty, 1}$ since $h \in \calH_B$ and $\|\blf_t\|_1 \leq L$.
This completes the reduction from $\olp$ to $\olpLO$. Next, we analyze the corresponding regret.

First, we observe that
\begin{align}
\max_{\widetilde{\btheta} \in {\ball}_{1, \infty}} \sprod{ \widetilde{\btheta}}{\sum_{t=1}^T \widetilde{\blf}_t} &= \max_{h \in \calH} \|\sum_{t=1}^T \widetilde{\blf}_t(h)\|_{1} = \max_{h \in \calH} \left\|\sum_{t=1}^T \frac{h(\x_t)}{B \cdot L}\cdot {\blf}_t\right\|_{1}\notag\\
&= \frac{1}{B\cdot L} \left(\max_{h \in \calH\,,\,\btheta \in \ball_{\infty}}\sprod{\btheta}{\sum_{t=1}^T h(\x_t)\cdot \blf_t}\right).\label{eq:max-eqv}
\end{align}
where in the final equality we used the definition of the dual norm.
Furthermore, observe that for any $\widetilde{\btheta}_t \in {\ball}_{1, \infty}$ and $\widetilde{\blf}_t = \frac{1}{B\cdot L } (h^{(1)}(\x_t) \cdot \blf_t, \cdots, h^{(|\calH|)}(\x_t)\cdot \blf_t)$, we have
\begin{align}
\E[\sprod{\btheta_t}{h_t(\x_t)\cdot\blf_t}] &= B\cdot L\sum_{h \in \calH} \gamma_t(h) \cdot \sprod{\frac{\widetilde{\btheta}_t(h)}{\gamma_t(h)}}{\frac{h(\x_t)}{B\cdot L}\cdot\blf_t}\notag \\  
&= B\cdot L \sum_{h \in \calH}  \sprod{ \widetilde{\btheta}_t(h)}{\widetilde{\blf}_t(h)} = B\cdot L \sprod{\widetilde{\btheta}_t}{\widetilde{\blf}_t} \label{eq:reward-eqv}.
\end{align}
where $\gamma_t(h) = \|\widetilde{\btheta}_t(h)\|_{\infty}$ denotes the probability of selecting group $h$ in round $t$ and $\btheta_t = (1/\gamma_t(h))\cdot\widetilde{\btheta}_t(h_t)$.
On combining \eqref{eq:max-eqv} and \eqref{eq:reward-eqv}, and taking expectations, we can conclude that:
\begin{align}
R_T(\mathcal{L}; \calH, \bTheta) &= \max_{h \in \calH, \, \btheta \in \ball_{\infty}} \sprod{ {\btheta}}{\sum_{t=1}^T h(\x_t)\cdot{\blf}_t} - \sum_{t=1}^T \E[\sprod{\btheta_t}{h_t(\x_t)\cdot\blf_t}] \notag \\
&= B \cdot L \left(\max_{\widetilde{\btheta} \in {\ball}_{1, \infty}} \sprod{ \widetilde{\btheta}}{\sum_{t=1}^T \widetilde{\blf}_t} - \sum_{t=1}^T \E\left[\sprod{ \widetilde{\btheta}_t}{ \widetilde{\blf}_t \rangle}\right]\right) = B \cdot L \cdot R_T^{\olpLO}(\widetilde{\mathcal{L}}; {\ball}_{1, \infty}). \label{eq:regret-eqv}
\end{align}
This completes the proof of the lemma.
\end{proof}

We now present a no-regret algorithm for $\olpLO$ whose runtime per round is linear in $|\calH|$.

\paragraph{Algorithm.} 
Our algorithm  for $\olpLO$ relies on two components: (i) a reward-maximizing no-regret online linear optimization algorithm, denoted by $\A$ (e.g., online gradient descent), and (ii) a reward-maximizing no-regret algorithm for the experts setting, denoted $\cal{E}$ (e.g., multiplicative weights update). 
Our algorithm will execute $|\calH|$ copies of $\A$, one for each $h \in \calH$; that is,  we treat each function $h$ as an expert who runs their own copy of $\A$, denoted $\A^h$, 
each measuring regret against
the action set $\ball_{\infty}$. 
Subsequently, the predictions made by $\A^h$ are aggregated using $\cal{E}$.
At a high level, each of the OLO algorithms $\A^h$ executes the action we would want to take if we were only concerned about calibration with respect to the specific hypothesis $h$, and the experts algorithm $\cal{E}$ selects for hypotheses that seem ``more difficult" with respect to calibration error (where our proxy for calibration error is precisely the cumulative reward of the corresponding OLO algorithm).
See~\Cref{alg:olp-lin} for a formal description. 

\begin{algorithm}[H]
\caption{\textsc{ Linearized Online Linear-Product Optimization}}
\label{alg:olp-lin}
\begin{algorithmic}[1]
\State {\bf Input:} OLO algorithm $\A$, experts algorithm $\cal{E}$
\State start instances $\A^1, \ldots, \A^{|\calH|}$, one for each $h \in \calH$
\For{$t = 1, \ldots T$}
\State obtain predictions from $\A^1, \ldots, \A^{|\calH|}$, denoted as $\btheta_t^h$ for $h \in \calH$
\State query $\cal{E}$ to obtain distribution $\bg_t$
\State let $h_t = h$ with probability $\bg_t(h)$ and predict $\btheta_t = \btheta_t^{h_t}$
\State observe reward vector $\widetilde{\blf}_t$ 
\State for $h \in \calH$, pass reward vector $\widetilde{\blf}_t(h)$ to $\A^h$
\State for $h \in \calH$, pass reward  $\langle{\btheta_t^h},{\widetilde{\blf}_t(h)}\rangle$ to $\cal{E}$
\EndFor
\end{algorithmic}
\end{algorithm}

In \Cref{alg:olp-lin}, we will use the popular \emph{online gradient descent} algorithm as our reward-maximizing no-regret online linear optimization algorithm, which is known to have regret $O(DG/\sqrt{T})$ where $D$ denotes the $\ell_2$ diameter of the action set and $G$ denotes the $\ell_2$-norm of the largest reward vector~\cite{hazan2016introduction}. 
For the experts algorithm $\mathcal{E}$, we use the familiar \emph{multiplicative weights update} algorithm which provides regret $O(\rho\sqrt{Tn})$ when rewards are bounded in $[-\rho, \rho]$ and $n$ denotes the number of experts. Consequently, we obtain the following guarantee.

\begin{lemma}\label{lem:no-regret-olp-lin}
    The regret $R_T^{\olpLO}(\widetilde{\mathcal{L}}; {\ball}_{1, \infty})$ of \Cref{alg:olp-lin} can be bounded as follows.
    \[
    \textstyle
    R_T^{\olpLO}(\widetilde{\mathcal{L}}; {\ball}_{1, \infty}) \leq \sqrt{T\log|\calH|} + \sqrt{TM} .
    \]
\end{lemma}

\begin{proof}
To bound the regret, we observe
    \begin{equation}\label{eq:mwu}
        \sum_{t=1}^T \E\left[\sprod{ \widetilde{\btheta}_t}{ \widetilde{\blf}_t \rangle}\right] = \,\, \sum_{t=1}^T \sum_{h \in \calH} \gamma_t(h) \cdot \sprod{\btheta_t^{h}}{\widetilde{\blf}_t(h)} \,\, \geq \,\, \max_{h \in \calH} \left\{ \sum_{t=1}^T \sprod{\btheta_t^{h}}{\widetilde{\blf}_t(h)} \right\} - \sqrt{T \log |\calH|} ,
    \end{equation}
    where the inequality follows by applying the no-regret property of the multiplicative weights update algorithm $\cal{E}$ and noting that $\langle{\btheta_t^{h}},{\widetilde{\blf}_t(h)}\rangle \in [-1, 1]$ for all $h \in \calH$ (since $\widetilde{\blf}_t \in \ball_{\infty, 1}$). 
    Subsequently, we apply the no-regret property of each of the online gradient descent algorithms $\mathcal{A}_{1}, \ldots, \cal{A}_{|\calH|}$ to obtain
    \begin{equation}\label{eq:ogd}
        \sum_{t=1}^T \sprod{\btheta_t^{h}}{\widetilde{\blf}_t(h)}  \,\, \geq \,\, \max_{\btheta \in B_{\infty}} \left\{  \sum_{t=1}^T \sprod{\btheta}{\widetilde{\blf}_t(h)}\right\}  - \sqrt{TM} \,\, = \,\, \| \sum_{t=1}^T \widetilde{\blf}_t(h) \|_1 - \sqrt{TM},
       \end{equation}
       where the regret bound uses the fact that $\| \widetilde{\blf}_t(h) \|_2 \leq \| \widetilde{\blf}_t(h) \|_1 \leq 1$ and that the $\ell_2$ diameter of $\ball_{\infty}$ is $\sqrt{M}$.
       Combining \eqref{eq:mwu} and \eqref{eq:ogd} gives
       \[
       \sum_{t=1}^T \E\left[\sprod{ \widetilde{\btheta}_t}{ \widetilde{\blf}_t \rangle}\right] \,\, \geq \,\, \max_{h \in \calH}  \| \sum_{t=1}^T \widetilde{\blf}_t(h) \|_1 - \sqrt{T\log|\calH|} - \sqrt{TM}, 
       \]
       which upon re-arranging proves the lemma.
\end{proof}

We can now complete the proof of \Cref{thm:main-inf}.
\begin{proof}[Proof of \Cref{thm:main-inf}.]
     Using \Cref{thm:reduction}, \Cref{clm:olp-reduction} and \Cref{lem:no-regret-olp-lin},  
    \begin{align*}
     \E[K(\pi_T, \calH)] &\leq \,\, \frac{B}{m} + \frac{R_{T}(\mathcal{L}; \calH)}{T} + 4B\sqrt{\frac{m\log(6T|\calH|)}{T}} + \frac{4mB}{T}\\
     & \leq \,\, \frac{B}{m} + B\cdot L \left(\sqrt{\frac{2m}{T}} + \sqrt{\frac{\log|\calH|}{T}}\right) + 4B\sqrt{\frac{m\log(6T|\calH|)}{T}} + \frac{4mB}{T} \\
        &\leq \,\, 7BT^{-1/3}\sqrt{\log(6T|\calH|)} + 5BT^{-1/2}\sqrt{\log|\calH|} = O\left(BT^{-1/3}\sqrt{\log(6T|\calH|)}\right),
    \end{align*}
    where the penultimate inequality uses $M \leq 2m$, and the final inequality uses $L \leq 1$ as per the reduction in 
    \Cref{thm:reduction} and  $m=T^{1/3}$.
\end{proof}


\subsection{Extending to Infinitely Many Groups}\label{sec:olp-infinite}
We now describe the necessary changes to obtain improved bounds for online multicalibration for infinite-sized hypothesis class $\calH$.
At a high-level, our approach reduces to the finite setting by constructing an appropriate \emph{covering} of the hypothesis class $\calH$.
Then, it uses the simple fact that the online multicalibration error is $1$-Lipschitz w.r.t. $\calH$. 

Recall our definition of covering with respect to the $L_{\infty}$ metric (\Cref{def:covering}).
We say that a subset $\calH_{\beta} = \{h^{1},\dots,h^{N}\} \sse \calH$ is a $\beta$-cover with respect to the $L_{\infty}$ metric if for every $h \in \calH$, there exists some $i \in [N]$ such that $\|h - h^i\|_{L_{\infty}} = \max_{x \in \calX} |h(x) - h^i(x)| \leq \beta$.
Henceforth, we denote $\calH_{\beta}$ and $|\calH_{\beta}|$ as the minimal $\beta$-cover and the corresponding $\beta$-covering number respectively of $\calH$, referring to the $\ell_{\infty}$ metric by default. With this definition in hand, the key steps for our extension to an infinite-sized hypothesis class $\calH$ are as follows:
\begin{enumerate}
    \item We begin by replacing  $\calH$ with $\calH_{\beta}$ for some $\beta > 0$ and subsequently appealing to \Cref{thm:main-inf} to bound the online multicalibration error with respect to  $\calH_{\beta}$.
    
    \item Then, by showing that the online multicalibration error is $1$-Lipschitz with respect to $\calH$ in the $L_{\infty}$ metric, we can conclude that the prediction is indeed multicalibrated with respect to $\calH$ up to the additional error term $\beta$ (see Lemma~\ref{lem:l1-lipschitz}).
    
    \item Finally, we appropriately set $\beta = \frac{o(T)}{T}$ (in a manner that optimally balances the terms $\sqrt{\log |\calH_{\beta}|}$ and $\beta$ in the multicalibration error). 
    We instantiate the last step with various examples of linear, polynomial, and uniformly Lipschitz convex function classes. 
\end{enumerate}

The main result of this section shows that \emph{any} algorithm that is multicalibrated with respect to $\calH_{\beta}$ is also multicalibrated with respect to $\calH \sse \calH_B$ up to an additional error of $\beta$.
\begin{lemma}\label{lem:l1-lipschitz}
    Let $\calH$ denote a collection of real-valued functions $h: \calX \to \R$ 
    and let $\calH_{\beta}$ denote a $\beta$-cover of $\calH$ with respect to the $L_{\infty}$ metric.
    If a sequence of predictions is multicalibrated with respect to $\calH_{\beta}$, then it is also multicalibrated with respect to $\calH$, up to an additive $\beta$. Formally, if
    $K(\pi_T, \calH_{\beta}) \leq \alpha$, then $K(\pi_T, \calH) \leq \alpha + \beta$.
\end{lemma}
\begin{proof}
Fix a transcript $\pi_T = \{(\x_t, p_t, y_t)\}_{t=1}^T$. 
To avoid notational clutter, we will drop the argument $\pi_T$ in the remainder of the proof.
We will also index predictions as $p \in \calP$, and use $S(p)$ to denote the set of rounds in which the prediction was $p$; that is, $S(p) = \{ t \in [T]: p_t = p\}$.
For $h \in \calH$, recall 
\[
K(h) = K(\pi_T, h) = \frac{1}{T} \sum_{p \in \calP} \left|\sum_{t \in S(p)} h(\x_t) \cdot (y_t - p)\right|.
\]
Furthermore, define $K_p(h) = \left|\sum_{t \in S(p)} h(\x_t) \cdot (y_t - p)\right|$ for every $p \in \calP$. 

\begin{claim}\label{claim:l1-lips}
For any $p \in \calP$, let $T_p = |S(p)|$. Then, we have for any $h, h' \in \calH$:
\[
| K_p(h) - K_p(h')| \,\, \leq \,\, T_p \cdot \|h - h'\|_{\ell_{\infty}}.
\]
\end{claim}

Using this claim, we can complete the proof of \Cref{lem:l1-lipschitz}. First, we note that
$K(h) = \frac{1}{T}\sum_{p \in \calP} K_p(h).$
Since $K(h) \geq 0$, we have $K(h) = |K(h)|$. 
Using this, and the triangle inequality twice, we get
\begin{align*}
K(h) - K(h') \,\,&=\,\, |K(h)| - |K(h')| \,\, \leq \,\, \left|K(h) - K(h')\right| \\
&=\,\, \left| \frac{1}{T} \sum_{p \in \calP} (K_p(h) - K_p(h')) \right| \,\, \leq \,\,  \frac{1}{T} \sum_{p \in \calP} \left|K_p(h) - K_p(h') \right| \\
&\leq \,\, \frac{1}{T} \sum_{p \in \calP} T_p \cdot \|h - h'\|_{\ell_{\infty}} \,\, = \,\, \|h - h'\|_{\ell_{\infty}},
\end{align*}
where the final inequality uses \Cref{claim:l1-lips} and the final equality used $\sum_{p \in \calP} T_p = T$. By symmetry, we similarly have
\[
K(h') - K(h) \leq \|h - h'\|_{\ell_{\infty}},
\]
from which we conclude that $|K(h) - K(h')| \leq \|h - h'\|_{\ell_{\infty}}$.
Now, fix an $h \in \calH$, and pick $h' \in \calH_{\beta}$ such that $\|h - h'\|_{L_{\infty}} \leq \beta$ (note that such an $h' \in \calH_{\beta}$ always exists according to Definition~\ref{def:covering}). Thus, we have 
\[
K(h) \,\, \leq \,\, K(h') + \|h - h'\|_{\ell_{\infty}} \,\, \leq \alpha + \beta, 
\]
as desired. We complete the proof of \Cref{lem:l1-lipschitz} by proving \Cref{claim:l1-lips}.
\begin{proof}[Proof of \Cref{claim:l1-lips}.]
Applying the reverse triangle inequality, we have
\begin{align*}
    K(h) - K(h') &= \left|\sum_{t \in S(p)}h(\x_t) \cdot (y_t - p) \right| - \left|\sum_{t \in S(p)}h'(\x_t) \cdot (y_t - p)\right| \\ 
    &\leq \left|\sum_{t \in S(p)}h(\x_t) \cdot (y_t - p) - \sum_{t \in S(p)}h'(\x_t) \cdot (y_t - p)\right| \\
    &=\left|\sum_{t \in S(p)} (h(\x_t)-h'(\x_t)) \cdot (y_t - p) \right|.
\end{align*}
Applying again the triangle inequality to the expression above, and using the fact that $y_t \in [0, 1]$ for all $t \in [T]$, we get
\begin{equation}\label{eq:cl1-a}
    K(h) - K(h') \leq \sum_{t \in S(p)} \left|(h(\x_t)-h'(\x_t)) \right| \leq T_p \cdot \max_{\x \in \calX} |h(\x) - h'(\x)| = T_p \cdot \|h - h'\|_{\ell_{\infty}}.
\end{equation}

\noindent By symmetry, we also have the following.
\begin{equation}\label{eq:cl1-b}
    K(h') - K(h) \leq T_p \cdot \|h' - h\|_{L_{\infty}}.
\end{equation}

\noindent Combining \eqref{eq:cl1-a} and \eqref{eq:cl1-b} completes the proof.
\end{proof}

This completes the proof of \Cref{lem:l1-lipschitz}.
\end{proof}

We can now complete the proof of \Cref{thm:main-inf-infinite} using \Cref{thm:main-inf} and \Cref{lem:l1-lipschitz}.
\begin{proof}[Proof of \Cref{thm:main-inf-infinite}.]
    Fix an $h \in \calH$, and let $h' \in \calH_{\beta}$ such that $\max_{\x \in \calX} |h(\x) - h'(\x)| \leq \beta$. 
    \begin{align*}
        \E[K(\pi_T, h)] \,\, &= \,\, \E[K(\pi_T, h')] \,\, + \,\, \E[K(\pi_T, h) - K(\pi_T, h')]   \\
        &\leq \,\, O\left(BT^{-1/3}\sqrt{\log(6T|\calH_{\beta}|)}\right) \,\, + \,\, \beta, 
    \end{align*}
    where the first term is bounded by  \Cref{thm:main-inf} on  the class $\calH_{\beta}$ and the second term follows from \Cref{lem:l1-lipschitz} since $\calH_{\beta}$ is a $\beta$-cover.
\end{proof}

As a consequence of this result, we obtain the following applications (by considering specific function classes for $\calH$).

\paragraph{Application 1: Polynomial Regression.}  
Suppose $\calX = [0, 1]^d$, and our hypothesis class corresponds to multivariate polynomial regression functions of degree $k$; that is, given $x \in \calX$,  $h(x) = \sum_{i = 1}^d\sum_{a = 1}^{k} h_{i + (a-1)k} \cdot x_i^{a}$ where $g \in \R^{kd}$.
Suppose that $\calH(d,k) = \{h \in \R^{kd}: \| h \|_1 \leq B\}$.
Then, constructing an $\beta$-cover for $\calH$ in the traditional sense gives us an
$\beta$-cover for $\calH$ in the functional sense.
In particular, since $\calX  = [0, 1]^d$, we have
\begin{align*}
    \max_{x \in \calX} \left| h(x) - h'(x) \right| \,\, \leq \,\, \| h - h' \|_1.
\end{align*}
It is then a standard fact that the $\beta$-covering number of the $\ell_1$ ball of radius $B$ with respect to the $\ell_1$-norm is at most $\left(1 + \frac{2B}{\beta}\right)^{kd}$.
Therefore, applying Theorem~\ref{thm:main-inf-infinite} and setting $\beta := T^{-1/3}$ gives us multicalibration error
\[
\E[K(\pi_T, \calH_{\mathsf{poly}(d,k)})] = \mathcal{O}\left(B(kd)^{1/2}T^{-1/3}\log(BT)\right).
\]

The class of bounded linear functions is subsumed by this class (when $k=1$), for which we obtain the \Cref{cor:linear-f}.

\paragraph{Application 2: Bounded, uniformly Lipschitz, convex functions.}

Let $\calH([0,1]^d, B, L) \subseteq \calH_B$ denote the set of real-valued convex functions defined on $\calX := [0, 1]^d$ that are uniformly bounded by $B$ and uniformly Lipschitz\footnote{Without the Lipschitz assumption,~\cite{guntuboyina2012covering} showed that the covering number of real-valued, bounded convex functions is actually infinity.} with constant $L$.
~\cite[Theorem 6]{bronshtein1976varepsilon} shows that in this case, $ |\calH_{\beta}| \leq C \left(\frac{1}{\beta}\right)^{d/2}$, where $C$ is a positive constant that depends only on the Lipschitz parameter $L$.
Therefore, applying Theorem~\ref{thm:main-inf-infinite} and setting $\beta := T^{-\frac{1}{2 + d/2}}$ yields multicalibration error
\[
K(\pi_T, \calH([0,1]^d, B, L)) = \widetilde{\mathcal{O}}(Bd^{1/2}T^{-1/3}\log(BT) + T^{-\frac{1}{2 + d/2}}).
\]
The second term dominates if $d > 2$, and reflects the standard ``curse of dimensionality" that we encounter in statistical learning of real-valued convex functions.


\section{Oracle Efficient Online Multicalibration}
\label{sec:oracle-olp}

In this section, we explore approaches to computationally efficient online multicalibration and ultimately prove Theorem~\ref{thm:main-oracle}.
We do this by 
adopting the framework of \emph{oracle efficiency}~\cite{hazan2016computational,syrgkanis2016efficient,daskalakis2016learning} for the online linear-product optimization problem ($\olp$) arising from the reduction in \Cref{thm:reduction}. 
In particular, we design an online learning algorithm for $\olp$ that makes a single call to an optimization oracle per round\footnote{Note that we can think of oracle calls requiring $O(1)$ time, as we have dispatched computational effort to the optimization oracle.} and does not need to access the augmented $\olpLO$ structure.
We recall the definition of our oracle.
\begin{definition}[Offline Oracle]\label{def:oracle}
    The offline oracle receives a sequence of contexts $\{\x_s\}_{s = 1}^t$ with corresponding reward vectors $\{\blf_s\}_{s = 1}^t$, coefficients $\{\kappa_s\}_{s=1}^t$, and an error parameter $\epsilon > 0$, and returns a solution $({h}^*, {\btheta}^*) \in \calH \times \Theta$ such that
\begin{align}\label{eq:oracle}
    \sum_{s=1}^t  \kappa_s \sprod{{\btheta}^*}{{h}^*(\x_s)\cdot \blf_s} \geq \max_{h \in \calH, \btheta \in \Theta} \left\{\sum_{s=1}^{t}  \kappa_s \sprod{\btheta}{h(\x_s) \cdot\blf_s}\right\} - \epsilon
\end{align}

Note that the oracle can handle input of variable length, i.e. the number of rounds $t$ itself is an implicit input to the optimization oracle.
\end{definition}
In the introduction, we claimed that solving~\eqref{eq:oracle} essentially amounts to evaluating approximate multicalibration error.
To see this, suppose we call the oracle with input length equal to $T$ and set $\kappa_t = 1$ for all $t \leq T$.
Then, for a fixed $h \in \calH$, we have
\[
\max_{\btheta \in \Theta} \sum_{t=1}^{T} \sprod{\btheta}{h(\x_t) \cdot\blf_t} = \max_{\btheta \in \Theta}  \sprod{\btheta}{\sum_{t=1}^{T} h(\x_t) \cdot\blf_t} =  \left\|\sum_{t=1}^T h(\x_t) \cdot \blf_t\right\|_1,
\]
and the argument $\btheta$ that maximizes the above is expressible in closed form. Maximizing this quantity over $h\in \calH$ yields the approximation to the multicalibration error with respect to the family $\calH$ that is defined in Equation~\eqref{eq:exp-l1-2} (where the approximation error is upper bounded in Lemma~\ref{lem:good-event}).

In the remainder of this section, we design an oracle-efficient online learning algorithm for $\olp$ that relies on the oracle defined in \eqref{eq:oracle}. 
In \Cref{sec:oracle-apps}, we provide explicit efficient constructions of the oracle under the assumption that the contexts $\{\x_t\}_{t=1}^T$ satisfy either the \emph{transductive} or \emph{small-separator} conditions, both of which are commonly made assumptions in oracle-efficient online learning~\cite{syrgkanis2016efficient,dudik2020oracle}.
We restate these assumptions for completeness in Section~\ref{sec:oracle-apps}.

To keep this section self-contained, let us recall the definition of $\olp$.
Let $\calX$ denote the context space, and let $\calH \times B_{\infty}$ denote the decision set where $\calH \sse \calH_B$ is a collection of real-valued functions $h: \calX \to \R$.
In round $t \in [T]$, the learner plays $(h_t, \btheta_t) \in \calH \times B_{\infty}$. The adversary reveals a context $\x_t$ and a reward vector $\blf_t$, which results in a reward of $\sprod{\btheta_t}{h(\x_t)\cdot \blf_t}$ to the learner. 
The goal of the learner (or an online learning algorithm) $\mathcal{L}$ is to minimize regret, which is defined as
\[
R_{T}(\mathcal{L}; (\x_1, \blf_1), \ldots, (\x_{T}, \blf_T); \calH) = \max_{h^* \in \calH, \btheta^* \in B_{\infty}} \left\{ \sprod{\btheta^*}{\sum_{t=1}^T  h^*(\x_t)\cdot\blf_t} \right\} - \sum_{t=1}^T \E[\sprod{\btheta_t}{h_t(\x_t)\cdot\blf_t}].
\]
Furthermore, recall that 
$\|\blf_t\|_1 \leq 1$ for all $t \in [T]$. 
We make a further observation tailored to the decision set $B_{\infty}$: for any $h \in \calH$, we have 
\[{\arg \max}_{\btheta \in B_{\infty}} \left\{\sum_{t=1}^{T} h(\x_t)\cdot \sprod{\blf_t}{\btheta}\right\}  \in \PMTheta.
\]
Based on this observation, we restrict our decisions to the Boolean hypercube; i.e., $\btheta_t \in \PMTheta$, and note that for any learning algorithm $\mathcal{L}$ for $\olp$, the regret remains unchanged. 
We also allow actions $\btheta_t$ to be randomized; as a result, $\E[\btheta_t]$ can and will lie in the interior of $\PMTheta$. 
Accordingly, we will assume the optimization oracle returns a solution $({h}^*, {\btheta}^*) \in \calH \times \PMTheta$ satisfying \eqref{eq:oracle}.

\subsection{Online Multicalibration via Admissibility and Implementability}\label{sec:oracle-multigroup-gamma}
We now give an algorithm for $\olp$ that is situated in the \emph{generalized Follow-the-Perturbed-Leader} (GFTPL) framework~\cite{dudik2020oracle}, and describe the conditions under which it can be made efficient with respect to the oracle in \Cref{def:oracle}.
In particular, our algorithm will ultimately require only a single oracle call per round. The algorithm is roughly constructed in two steps.

\begin{enumerate}
\item The first step is to design an algorithm in the GFTPL framework that obtains sublinear regret.  
The main idea is to draw a lower-dimensional random vector $\bm{\alpha} \in \R^N$ where each $\alpha_j$ is drawn independently from an appropriate distribution, and $N \ll |\calH|$---essentially, we will think of runtime that is linear in $N$ to be acceptable.
The payoff of each of the algorithm's actions is perturbed by some linear combination of $\bm{\alpha}$ as given by a \emph{perturbation translation matrix} $\bm{\Gamma} \in [-B, B]^{(|\calH| \times 2^M) \times N}$ (note that $|\calH| \times 2^M$ denotes the size of the action space of the algorithm). In each round $t \in [T]$, the algorithm selects $(h_t, \btheta_t) \in \calH \times \PMTheta$ such that 
    \begin{equation}\label{eq:oracle-ptm}
        \sum_{s=1}^{t-1}  \sprod{\btheta_t}{h_t(\x_s)\cdot\blf_s} + \bm{\alpha} \cdot \bm{\Gamma}_{(h_t,\btheta_t)} \geq 
        \max_{h \in \calH, \btheta \in \PMTheta} \left\{\sum_{s=1}^{t-1} \sprod{\btheta}{h(\x_s)\cdot \blf_s} + \bm{\alpha} \cdot \bm{\Gamma}_{(h,\btheta)} \right\} - \epsilon
    \end{equation}
    for some accuracy parameter $\epsilon > 0$. We will show that is equivalent to the oracle~\eqref{eq:oracle} in Definition~\ref{def:oracle} (refer to the ``implementability condition in \Cref{def:implementability}). 
    This algorithm is described in \Cref{alg:olp-gftpl}, and is no-regret up to the error of the oracle $\epsilon$ as long as $\bm{\Gamma}$ satisfies the following $\delta$-admissibility condition.

\begin{definition}[$\delta$-admissibility~\cite{dudik2020oracle}]
    A translation matrix $\bm{\Gamma}$ is $\delta$-admissible if its rows are distinct, and distinct elements within each column differ by at least $\delta$. Formally, in our framework, for every pair $(h,\btheta) \neq (h',\btheta')$ we need $\bm{\Gamma}_{(h,\btheta)} \neq \bm{\Gamma}_{(h',\btheta')}$ and for every $j \in [N]$ we either have $\Gamma_{(h,\btheta),j} = \Gamma_{(h',\btheta'),j}$ or $|\Gamma_{(h,\btheta),j} - \Gamma_{(h',\btheta'),j}| \geq \delta$.
\end{definition}

\begin{algorithm}
\caption{Generalized FTPL for $\olp$}
\label{alg:olp-gftpl}
\begin{algorithmic}[1]
\State \textbf{Input:} perturbation matrix $\bm{\Gamma} \in [0, 1]^{(|\calH| \times 2^M) \times N}$,  and accuracy $\epsilon > 0$
\State generate randomness $\bm{\alpha} = (\alpha_1, \cdots, \alpha_N)$ where $\alpha_i \sim \text{Unif}[0, \sqrt{T}]$   
\For{$t = 1, \ldots, T$}
\State select $(h_t, \btheta_t) \in \calH \times \PMTheta$ such that 
    \[
        \sum_{s=1}^{t-1}  \sprod{\btheta_t}{h_t(\x_s)\cdot\blf_s} + \bm{\alpha} \cdot \bm{\Gamma}_{(h_t,\btheta_t)} \geq 
        \max_{h \in \calH, \btheta \in \{\pm 1\}^M} \left\{\sum_{s=1}^{t-1} \sprod{\btheta}{h(\x_{s})\cdot \blf_s} + \bm{\alpha} \cdot \bm{\Gamma}_{(h,\btheta)} \right\} - \epsilon
    \]
\State observe context $\x_t$ and reward vector $\blf_t$, and receive payoff $\sprod{\btheta_t}{h_t(\x_t) \cdot \blf_t}$
\EndFor
\end{algorithmic}
\end{algorithm}

\item It turns out that the aforementioned $\delta$-admissibility condition suffices to design an algorithm that obtains an overall regret guarantee of $O(N\sqrt{T}/\delta + \epsilon T)$. 
However, this does not yet guarantee oracle-efficiency since a naive construction of $\bm{\Gamma}$ would require space that is linear in $| \calH |$ and exponential in $M$. 
Thus, we need the property that the perturbations to each action $(h, \btheta)$ can be simulated efficiently through the optimization oracle~\eqref{eq:oracle}; that is, without needing to actually access $\bm{\Gamma}$.
This requirement is captured by the \emph{implementability} condition, stated next. 

\begin{definition}[Implementability]\label{def:implementability}
A translation matrix $\bm{\Gamma}$ is implementable with complexity $D$ if for any realization of the random vector $\bm{\alpha}$, there exists a set of coefficients $\{\kappa_{j}\}_{j=1}^D$ and vectors $\{\widehat{\blf}_j\}_{j= 1}^D$ (both of which will depend on $\bm{\alpha}$) such that we can express 
\begin{align}\label{eq:implementability}
\bm{\alpha} \cdot \bm{\Gamma}_{(h, \btheta)} = \sum_{j = 1}^{D} \kappa_j \cdot \sprod{\btheta}{h(\x_j) \cdot \widehat{\blf}_j}
\end{align}
for all  $(h, \btheta) \in \calH \times \PMTheta.$
\end{definition}
Crucially, this allows us to simulate \eqref{eq:oracle-ptm} by \eqref{eq:implementability} (which is exactly the oracle in \Cref{def:oracle}).
We note that the definition differs slightly from the one provided in~\cite{dudik2020oracle} but the definitions can be shown to be equivalent, and this definition turns out to be slightly more convenient to work with in the proof.
\end{enumerate}

The main theorem of this section shows that the GFTPL framework achieves sublinear regret for $\olp$.

\begin{theorem}\label{thm:oracle-efficient-full}
Consider the instance of $\olp$ with the decision set $\calH \times \PMTheta$. 
Suppose that the perturbation translation matrix $\bm{\Gamma} \in [-B, B]^{(|\calH| \times 2^M) \times N}$ satisfies $\delta$-admissibility and satisfies implementability with complexity $D$. 
Then, Algorithm~\ref{alg:olp-gftpl} executed with $\epsilon = 1/\sqrt{T}$ satisfies the expected regret $R_T(\mathcal{L}; \calH)$ guarantee
\begin{equation}\label{eq:oracle-eff-regret}
\textstyle  R_T(\mathcal{L}; \calH) \leq O\left(\frac{B^2N\sqrt{T}}{\delta}\right).
\end{equation}
Furthermore, this algorithm requires a single oracle call per round, and the per-round complexity is $O(T + ND)$.
\end{theorem}

The proof of \Cref{thm:oracle-efficient-full} is provided in Appendix~\ref{app:oracle-efficient-proof}.
The key steps are similar to the analysis of the GFTPL algorithm in~\cite{dudik2020oracle}, i.e., the analysis includes characterizing an approximation error term and stability error term. We note that combining \eqref{eq:oracle-eff-regret} with \Cref{thm:reduction} gives a ``black-box'' algorithm for obtaining online $\ell_1$-multicalibration given a $\delta$-admissible and implementable $\Gamma$.

\subsection{Our admissible and implementable construction}\label{sec:oracle-apps}
In this section, we give settings in which one can explicitly construct a $\bm{\Gamma}$ matrix that is admissible and implementable. 
Specifically, we will give efficient algorithms for the \emph{transductive} setting and the \emph{small-separator} setting, defined next. 

We also assume in this section that $\calH$ is binary-valued, i.e., $h: \calX \to \{0,1\}$ for all $h \in \calH$.

\medskip
\noindent {\bf Transductive Setting.} In the transductive setting, the learner has access to the set $\calX$ from which contexts will be drawn. Formally, we will say that $\calX = \{\x^1, \ldots, \x^D\}$ and that at each round $t$, the context $\x_t \in \calX$.

\medskip
\noindent {\bf Small-Separator Setting.} In the small-separator setting, we
assume access  to a small
set of contexts $\calX$, called a \emph{separator}.   
Let $\calX = \{\x^1, \ldots, \x^D\}$, and we have 
that for any two groups $h,h' \in \calH$, there exists a feature $\bold{x} \in \calX$ such that $h(\x) \neq h'(\x)$.

\medskip
These settings are commonly used, even for the simpler problem of oracle-efficient online prediction~\cite{syrgkanis2016efficient,dudik2020oracle}, and have since been built on to ensure oracle-efficient online learning under smoothed data or transductive learning with hints\footnote{It is conceivable, but would take a non-trivial argument, to show that our approach could be extended to these more complex settings as well.}
~\cite{block2022smoothed,haghtalab2022oracle}.
We are now ready to define our construction of the translation matrix $\bm{\Gamma}$ for these two settings.
In particular, we prove the following. 

\begin{lemma}\label{lem:gamma-lemma}
    In the transductive and small-separator settings, there exists a perturbation translation matrix $\bm{\Gamma} \in [-B, B]^{(|\calH|\times 2^M) \times (DM)}$
    that is $1$-admissible and implementable with complexity $D$.
\end{lemma}
\begin{proof}
We consider $\bm{\Gamma} \in [-B, B]^{(|\calH| \times 2^M) \times (DM)}$, so that the rows of $\bm{\Gamma}$ can be indexed by $(h, \btheta) \in \calH \times \PMTheta$ and the columns of $\bm{\Gamma}$ can be indexed by $(j, i) \in [D] \times [M]$.
Then, we will specify each entry as $\bm{\Gamma}_{(h, \btheta), (i,j)} = h(\bold{x}^i) \cdot \theta_j$. Note that $|\bm{\Gamma}_{(h, \btheta), (i,j)}| \leq B$ since $h \in \calH \subseteq \calH_B$.
We first prove implementability and then $1$-admissibility.

\medskip
\noindent{\bf Proof of implementability.} Consider any realization of the noise $\bm{\alpha}$. As with the columns of $\bm{\Gamma}$, we index $\bm{\alpha}$ by the tuple $(j,i) \in [D] \times [M]$. We denote as $\bm{\alpha}_{(j,\cdot)} \in \R^M$ the sub-vector $\left(\alpha_{(j,1)}, \cdots, \alpha_{(j,M)}\right)$. Then, we have
\begin{align*}
    \sprod{\bm{\alpha}}{\bm{\Gamma}_{(h,\btheta)}} &= \sum_{j=1}^D \sum_{i=1}^M \alpha_{(j,i)} \cdot h(\bold{x}^j) \cdot \theta_i \\
    &= \sum_{j=1}^D h(\bold{x}^j) \sprod{\bm{\alpha}_{(j,\cdot)}}{\btheta},
\end{align*}
which clearly satisfies the implementability condition~\eqref{eq:implementability} with complexity $D$,  $\kappa_{j} = 1$ for all $j \in [D]$, and $\widetilde{\blf}_{j} = \bm{\alpha}_{(j,\cdot)}$.

\medskip
\noindent{\bf Proof of admissibility.} Consider any two rows indexed by $(h,\btheta)$, $(h',\btheta')$.
We need to show that:
\begin{enumerate}
    \item $\bm{\Gamma}_{(h,\btheta)} \neq \bm{\Gamma}_{(h',\btheta')}$.
    \item If $\Gamma_{(h,\btheta),(j,i)} \neq \Gamma_{(h',\btheta'),(j,i)}$, then $|\Gamma_{(h,\btheta),(j,i)} - \Gamma_{(h',\btheta'),(j,i)}| \geq 1$.
\end{enumerate}
Showing the first statement is easy: either $h \neq h'$, in which case there exists at least one index $j \in [D]$ such that $h(\bold{x}^j) \neq h'(\bold{x}^j)$.
In the other case where $h = h'$, consider any index $j$ such that $h(\bold{x}^j) \neq 0$ (at least one such index must exist, otherwise we can simply remove $h$ from the hypothesis class without loss of generality).
Further, since in this case we must have $\btheta \neq \btheta'$, there exists at least one index $i$ for which $\theta_i \neq \theta'_i$.
Therefore, we have $\bm{\Gamma}_{(h,\btheta),(j,i)} = h(\bold{x}^j) \theta_i \neq h(\bold{x}^j) \theta'_i$.

To show the second statement, we use the structure of the Boolean hypercube to say that if $\theta_i \neq \theta'_i$, then $|\theta_i - \theta'_i| = 2$.
Similarly, if $\theta_i \neq 0$, then $|\theta_i| = 1$.
Then, we again consider two cases when $\bm{\Gamma}_{(h,\btheta),(j,i)} \neq \bm{\Gamma}_{(h',\btheta'),(j,i)}$.
The first is if $h(\x^j) \neq h'(\x^j)$.
In this case, suppose without loss of generality that $h(\x^j) = 1$ while $h'(\x^j) = 0$.
Then, we have
\begin{align*}
    |\bm{\Gamma}_{(h,\btheta),(j,i)} - \bm{\Gamma}_{(h',\btheta'),(j,i)}| = |\theta_i| = 1.
\end{align*}
On the other hand, suppose that $h(\x^j) = h'(\x^j)$.
If $h(\x^j) = 0$, we have $\Gamma_{(h,\btheta), (j,i)} = \Gamma_{(h',\btheta'), (j,i)}$.
If $h(\bold{x}^j) = 1$, we have
\begin{align*}
    |\bm{\Gamma}_{(h,\btheta),(j,i)} - \bm{\Gamma}_{(h',\btheta'),(j,i)}| = |\theta_i - \theta'_i| \in \{0,2\}
\end{align*}
This proves the second statement and therefore we have proved $1$-admissibility, which completes the proof of the lemma.
\end{proof}

We now examine implications of Lemma~\ref{lem:gamma-lemma} for oracle-efficient multicalibration.
Plugging in $\delta = 1$ and $N = DM$ into the statement of Theorem~\ref{thm:oracle-efficient-full}, and appealing to the reduction in \Cref{thm:reduction} gives us
\begin{align*}
    \E[K(\pi_T,\calH)] \,\, &\leq \,\, \frac{B}{m} + O\left({BDM\sqrt{T}}\right) + 4B\sqrt{\frac{m\log(6T|\calH|)}{T}} + \frac{4mB}{T}.
\end{align*}
Finally, setting the allowable error in the optimization oracle to $\epsilon = \frac{1}{M} = \frac{1}{T^{1/4}}$ and using $M = m+1$ gives us
\begin{align*}
\E[K(\pi_T,\calH_B)] &\leq O\left(BD T^{-1/4}\sqrt{\log(T|\calH|)}\right),
\end{align*}
which is the bound stated in \Cref{sec:results}.

\paragraph{Potential improvements and extensions of analysis.} It is natural to ask about the extent to which the assumptions we have made --- access to the offline oracle, binary-valued $\calH$ and either transductive or sufficiently separated data --- can be weakened or improved.
We begin by noting that our proofs will directly work (at additional factor loss) given a ``$C$-approximate oracle'' since they rely on the GFTPL framework of \citet{dudik2020oracle} in a black-box manner. 
Two concurrent recent works~\cite{block2022smoothed,haghtalab2022oracle} showed that the generalized FTPL framework can provide oracle-efficient regret bounds for online supervised learning on infinite-sized hypothesis classes with bounded VC dimension (in the case of binary labels) or pseudodimension (in the case of real-valued labels) if the contexts are \emph{smoothed} with respect to some known base probability measure.
We believe that adapting their proof technique, especially for the stability term, to the $\olp$ instance is an interesting direction for future work.

At a more fundamental level, the discrepancy in rates between inefficient and oracle-efficient multicalibration arises solely from the linear dependence on $m$ (or $M$) in the regret bound of oracle-efficient $\olp$, as compared to the $O(m^{1/2})$ dependence in the inefficient $\olp$ implementation.
Improving the dependence on $m$ in the oracle-efficient regret bound would require improving the regret analysis of~\cite{dudik2020oracle}, which would be of independent interest.

\subsection*{Acknowledgments}
We are grateful to the anonymous reviewers of ICML for valuable feedback.
The last author is thankful to Marco Molinaro for discussions on online linear optimization  over general convex bodies and the proof of \Cref{lem:no-regret-olp-lin}.

\appendix
\section{Missing Proofs from Section \ref{sec:reduction}} 

\subsection{Proof of the Missing \Cref{lem:good-event}} \label{sec:lemmaVecAzuma}
The proof relies on a ``vector'' form of the Azuma-Hoeffding inequality  (see Theorem 1.8 in~\cite{hayes2005large}).

\begin{theorem}[Vector Azuma–Hoeffding's inequality]\label{thm:azuma}
    Let $\bold{S}_n = \sum_{t=1}^n \bold{X}_t$ be a martingale 
    relative to the sequence $\bold{X}_1, \ldots, \bold{X}_n$ where each $\bold{X}_t$ takes values in $\R^d$ and satisfies
    (i)~$\E[\bold{X}_t] = \bm{0}$ and (ii)~$\| \bold{X}_t \|_2 \leq c$.
    Then, for any $\eta > 0$ and $n \geq 1$, we have 
    \[
    \pr(\|\bold{S}_n\|_2 \geq \eta) \leq 2e^2\exp\left(\frac{-\eta^2}{2nc^2} \right).
    \]
\end{theorem}

Given this inequality, we can now complete the proof of \Cref{lem:good-event}. 

\goodevent*

\begin{proof}[Proof of \Cref{lem:good-event}]
    For each $h \in \calH$ and $t \in [T]$, we define the vector $\bold{Y}_t^h \in \R^{\M}$ such that 
    \[
    \bold{Y}_t^h(p) =\Big(\mathbb{I}\{p_t = p\} - \w_t(p)\Big) \cdot h(\x_t) \cdot (y_t - p) \enspace .
    \]
    We now verify that the conditions on the sequence $\{\bold{Y}_t^h(p)$ for the Vector Azuma-Hoeffding inequality (Theorem~\ref{thm:azuma}) hold.
    We note that 
    \begin{align*}
    \|\bold{Y}_t^h\|_2 &\,\,= \,\, \sqrt{\sum_{p \in \calP}\bold{Y}_t^h(p)^2 } \,\, = \,\, \sqrt{\sum_{p \in \calP} h(\x_t)^2 \cdot (y_t - p)^2 \cdot \left(\mathbb{I}\{p_t = p\} - \w_t(p)\right)^2} \\
    &\,\, \leq \max_{\x \in \calX}|h(\x)| \cdot \sqrt{\sum_{p \in \calP}\left(\mathbb{I}\{p_t = p\} - \w_t(p)\right)^2} \,\, \leq \,\, \sqrt{2}B,
    \end{align*}
    where the final inequality follows from the fact that $h \in \calH \subseteq \calH_B$, $\sum_{p \in \calP}\mathbb{I}\{p_t = p\} = 1$, and that $\w_t$ is a distribution.
    This verifies condition (ii).
    Next, observe that
    \begin{align*}
    \E[\bold{Y}_t^h(p) \mid \bold{Y}^h_1, \ldots, \bold{Y}^h_{t-1}] \,\, &= \,\, \E[h(\x_t) \cdot (y_t - p) \cdot \left(\mathbb{I}\{p_t = p\} - \w_t(p)\right) \mid \bold{Y}_1^h, \ldots, \bold{Y}_{t-1}^h] \\
    &= \,\, h(\x_t) \cdot (y_t - p) \cdot \E[\left(\mathbb{I}\{p_t = p\} - \w_t(p)\right) \mid\bold{Y}_1^h, \ldots, \bold{Y}_{t-1}^h] \,\, = \,\, 0
    \end{align*}
    since $\w_t$ only depends on information from rounds $1, \ldots, t-1$, and $\E[\left(\mathbb{I}\{p_t = p\}\right) \mid \bold{Y}_1^h, \ldots, \bold{Y}_{t-1}^h] = \w_t(p)$.
    This verifies condition (i).
    Thus, $\bold{S}_n^h = \sum_{t=1}^n \bold{Y}_t^h$ is a martingale with respect to $\bold{Y}_1^h, \bold{Y}_2^h, \ldots$, and 
    \[
     \|\bold{S}_T^h\|_1 = \sum_{p \in \calP}\left| \left(\sum_{t=1}^T \mathbb{I}\{p_t = p\} \cdot  h(\x_t) \cdot (y_t - p)\right) -  \left(\sum_{t=1}^T \w_t(p) \cdot  h(\x_t) \cdot (y_t - p) \right) \right|.
    \]
    Furthermore, since $\bold{S}_T^h \in \R^{\M}$, we have $\| \bold{S}_T^h \|_1 \leq \sqrt{\M}\cdot \| \bold{S}_T^h \|_2$, and applying \Cref{thm:azuma} implies that
    \begin{align*}
        \pr\left( \| \bold{S}_T^h \|_2 \geq 2\sqrt{2}B \sqrt{T \log(6T|\calH|)} \right) &\leq \,\, 2e^2\exp \left( \frac{-8B^2T\log(6T|\calH|)}{4B^2T} \right) \,\, = \,\, \frac{2e^2}{36(T)^2 |\calH|^2} \leq \frac{1}{T|\calH|}.
    \end{align*}
    Thus, with probability at least $1 - 1/(T|\calH|)$, we have 
    \begin{equation}\label{eq:good-event}
    \| \bold{S}_T^h \|_1 \leq \sqrt{\M}\cdot \| \bold{S}_T^h \|_2 \leq \sqrt{2m} \cdot 2\sqrt{2}B \sqrt{T\log(6T|\calH|)} \leq 4B \sqrt{Tm\log(6T|\calH|)},
    \end{equation}
    since $\M = m+1 \leq 2m$.
    We are now ready to complete the proof. Let $G_h$ denote the event that $\| \bold{S}_T^h \|_1 \leq 4B \sqrt{Tm \log(6T|\calH|)}$ for $h \in \calH$. 
    Let $G$ denote the event that $G_h$ holds \emph{simultaneously} for all $h \in \calH$; i.e. $G = \cap_{h \in \calH}G_h$. By the union bound over all $h \in \calH$, we can conclude that $\pr(G) \geq 1 - 1/T$.
    Then, we have
    \begin{align*}
        \E \left[\max_{h \in \calH} \left\{\| \bold{S}_T^h \|_1]\right\}\right] \,\,&= \,\, \E\left[ \max_{h \in \calH} \left\{ \| \bold{S}_T^h \|_1\right\} \mid G\right] \cdot \pr(G) \,\, + \,\, \E\left[ \max_{h \in \calH} \left\{\| \bold{S}_T^h \|_1\right\} \mid \overline{G}\right] \cdot \pr(\overline{G}) \\
        &\leq 4B\sqrt{Tm\log(6T|\calH|)} \cdot \left(1-\frac{1}{T}\right)  \,\, + \,\, \E\left[ \max_{h \in \calH} \left\{\| \bold{S}_T^h \|_1\right\} \mid \overline{G}\right]  \cdot \pr(\overline{G}) \\
        &\leq 4B\sqrt{Tm\log(6T|\calH|)}  \,\, + \,\, 2\M BT \cdot \frac{1}{T} \,\, \leq \,\, 4B\sqrt{Tm\log(6T|\calH|)}  \,\, + \,\, 4mB.
    \end{align*}
    where the first inequality uses \eqref{eq:good-event} for all $h \in \calH$ and the second inequality bounds $\| \bold{S}_T^h \|_1 \leq 2\M TB$ for all $h \in \calH$. The final inequality uses $\M = m+1 \leq 2m$. This completes the proof of the lemma.  
\end{proof}

\subsection{Constructing the Halfspace Oracle} \label{sec:reduction-halfspace}

Here we repeat the construction of an efficient  halfspace oracle from \cite{abernethy2011blackwell} for our setting.
In particular, we will design an efficient oracle $\cal{O}$ that can,
given $\x \in \calX$, $h \in \calH$ and $\btheta \in B_{\infty}$, select a distribution $\w \in \R^{\M}$ 
such that
    for all $y \in [0, 1]$, we have 
     \[ 
    \sum_{i=0}^m  \btheta(i) \cdot h(\x) \cdot \w(i) \cdot \left(y - \frac{i}{m}\right) \leq \frac{B}{m}.
    \]
We describe this oracle in \Cref{alg:l1-oracle}. 

\begin{algorithm}[H]
\caption{$\mathcal{O}(\x, h, \btheta)$}
\label{alg:l1-oracle}
\begin{algorithmic}[1]
\State \textbf{Input:} $\x \in \calX$, $h \in \calH \sse \calH_B$ and $\btheta: || \btheta ||_{\infty} \leq 1$
\State define $\widehat{\btheta}: \widehat{\btheta}(i) = \btheta(i)\cdot   h(\x)$ 
\If{$\widehat{\btheta}(0) \leq 0$}
\State $\w \gets \delta_0$
\ElsIf{$\widehat{\btheta}(m) \geq 0$}
\State $\w \gets \delta_m$
\Else
\State find coordinate $i$ such that $\widehat{\btheta}(i) > 0$ and $\widehat{\btheta}(i+1) \leq 0$
\State $\w \gets \frac{\widehat{\btheta}(i+1)}{\widehat{\btheta}(i+1)- \widehat{\btheta}(i)}\cdot \delta_i + \frac{\widehat{\btheta}(i)}{\widehat{\btheta}(i) - \widehat{\btheta}(i+1)}\cdot \delta_{i+1}$
\EndIf
\State \Return $\w$
\end{algorithmic}
\end{algorithm}

It immediately follows from the description that $\w$ is a valid distribution.
Furthermore, note that this oracle can be implemented efficiently. 
If $\widehat{\btheta}(0) > 0$ and $\widehat{\btheta}(m) < 0$, then there must exist $i$ such that $\widehat{\btheta}(i) > 0$ and $\widehat{\btheta}(i+1) \leq 0$, and such an index can be found using binary search.  Thus, this algorithm requires at most $O(\log(m))$ computations. 
The following lemma proves the main property of the halfspace oracle.

\begin{lemma}\label{lem:halfspace-oracle}
Given any $\x \in \calX$, $h \in \calH$ and $\btheta \in B_{\infty}$, let $\w = \mathcal{O}(\x, h, \btheta)$ be the output of \Cref{alg:l1-oracle}. Then, for any $y \in [0, 1]$, we have $\sum_{i=0}^{m} \btheta(i) \cdot h(\x) \cdot\left( y - \frac{i}{m} \right) \cdot \w(i)  \leq \frac{B}{m}$.
\end{lemma}
\begin{proof}
We first note that we can write 
\[
\sum_{i=0}^{m} \btheta(i) \cdot h(\x) \cdot\left( y - \frac{i}{m} \right) \cdot \w(i)  = \sum_{i=0}^{m} \widehat{\btheta}(i) \cdot \left( y - \frac{i}{m} \right) \cdot \w(i), 
\]
where we have used $\widehat{\btheta}(i) =  \btheta(i)\cdot h(\x) $ for all $i = 0, 1, \ldots, m$.
Observe that if $\widehat{\btheta}(0) \leq 0$ or $\widehat{\btheta}(m) \geq 0$, the lemma is trivially true.
Otherwise, we have
\begin{align*}
     \sum_{i=0}^{m} \widehat{\btheta}(i) \cdot \left( y - \frac{i}{m} \right) \cdot \w(i) &= \left( y - \frac{i}{m} \right) \cdot \w(i) \cdot \widehat{\btheta}(i) + \left( y - \frac{i+1}{m} \right) \cdot \w(i+1) \cdot \widehat{\btheta}(i+1) \\
    &= \left( y - \frac{i}{m} \right) \cdot \frac{\widehat{\btheta}(i+1)}{\widehat{\btheta}(i+1)- \widehat{\btheta}(i)} \cdot \widehat{\btheta}(i) + \left( y - \frac{i+1}{m} \right) \cdot \frac{\widehat{\btheta}(i)}{\widehat{\btheta}(i) - \widehat{\btheta}(i+1)} \cdot \widehat{\btheta}(i+1) \\
    &= -\frac{1}{m} \cdot \frac{\widehat{\btheta}(i) \cdot \widehat{\btheta}(i+1)}{\widehat{\btheta}(i) - \widehat{\btheta}(i+1)} \ \leq \frac{1}{m} \cdot \frac{\max\{|\widehat{\btheta}(i)|, |\widehat{\btheta}(i+1)|\}}{2} \leq \frac{B}{m},
\end{align*}
where the penultimate inequality uses the AM-HM inequality, and
the final inequality follows from the fact that $\btheta \in B_{\infty}$ and $\max_{\x \in \calX}|h(\x)| \leq B$ for all $h \in \calH$. 
\end{proof}

\section{Missing Proofs from \Cref{sec:oracle-olp}}

\subsection{Proof of \Cref{thm:oracle-efficient-full}}\label{app:oracle-efficient-proof}
We first show that if the perturbation translation matrix $\bm{\Gamma} \in [-B, B]^{(|\calH| \times 2^M) \times N}$ satisfies $\delta$-admissibility, then \Cref{alg:olp-gftpl} has regret bounded by $O(BN\sqrt{T}/\delta).$ 
Then, we show that since $\Gamma$ satisfies implementability with complexity $D$, we can indeed execute \Cref{alg:olp-gftpl} efficiently.
For ease of exposition, we will use $R_T$ to denote the regret of \Cref{alg:olp-gftpl}.
To begin, recall that we defined
\begin{align}\label{eq:regret}
R_T = \max_{h^* \in \calH, \btheta^* \in \PMTheta} \left\{ \sprod{\btheta^*}{\sum_{t=1}^T  h^*(\x_t)\cdot\blf_t} \right\} - \sum_{t=1}^T \E[\sprod{\btheta_t}{h_t(\x_t)\cdot\blf_t}].
\end{align}
It suffices to uniformly upper bound the regret incurred by each $h \in \calH$; we will provide a worst-case analysis of this quantity for arbitrary reward vectors $\{\blf_t\}_{t=1}^T$.
Formally, we can write
\begin{align*}
    R_T &= \max_{h \in \calH} R_T(h) \text{ where } \\
    R_T(h) &:= \max_{\btheta \in \PMTheta} \left\{ \sum_{t=1}^T  \sprod{\btheta}{h(\x_t)\cdot\blf_t} \right\} - \E\left[\sum_{t=1}^T \sprod{\btheta_t}{h_t(\x_t)\cdot \blf_t}\right].
\end{align*}

For each $h \in \calH$, let $\btheta_h^* \in \PMTheta$ be a maximizer of the function $\sum_{t=1}^T \sprod{\btheta}{h(\x_t)\cdot \blf_t}$ in the variable $\btheta$. Then, we can decompose $R_T(h)$ into two terms as below:
\begin{align}
    R_T(h) &= \E\left[\sum_{t=1}^T \sprod{\btheta_h^*}{h(\x_t)\cdot \blf_t}  - \sum_{t=1}^T  \sprod{\btheta_t}{h_t(\x_t)\cdot\blf_t}\right] \notag \\
    &= \E\Big[ \underbrace{\sum_{t=1}^T \sprod{\btheta_h^*}{h(\x_t)\cdot \blf_t} - \sum_{t=1}^T \sprod{\btheta_{t+1}}{h_{t+1}(\x_t)\cdot \blf_t}}_{T_1}\Big] \notag \\ 
    & \qquad \qquad + \underbrace{\E\Big[\sum_{t=1}^T \sprod{\btheta_{t+1}}{h_{t+1}(\x_t)\cdot \blf_t} - \sum_{t=1}^T \sprod{\btheta_t}{h_t(\x_t)\cdot\blf_t}\Big]}_{T_2} \label{eq:regret-h-decomp}
\end{align}
The first term $T_1$ corresponds to an \emph{approximation error term}: suppose we could use the clairvoyant decision $(h_{t+1},\btheta_{t+1})$ for round $t$ --- without noise, this would be optimal, but with noise it will create extra error. The following lemma(Lemma~B.1 in~\cite{dudik2020oracle}), characterizes $T_1$ pointwise for every realization of the noise $\bm{\alpha}$.
\begin{lemma}[Be-the-Approximate Leader Lemma]\label{lem:btl-lemma}
    In the Generalized FTPL algorithm, we have
    \[
    T_1 \leq \bm{\alpha} \cdot (\bm{\Gamma}_{(h_1,\btheta_1)} - \bm{\Gamma}_{(h,\btheta)}) + \epsilon \cdot (T+1)
    \]
    for each $(h, \btheta) \in \calH \times \PMTheta$ and every realization of the noise $\bm{\alpha}$.
\end{lemma}
\begin{proof}
We will prove this lemma by induction on $T$. For the base case, $T = 0$, and so it suffices to show that $\bm{\alpha} \cdot (\bm{\Gamma}_{(h_1,\btheta_1)} - \bm{\Gamma}_{(h,\btheta)}) + \epsilon \geq 0$ for all $(h,\btheta)$.
Here, the statement follows directly from the $\epsilon$-approximate optimality of oracle; that is, we select $(h_1, \btheta_1)$ such that
\[\bm{\alpha} \cdot \bm{\Gamma}_{(h_1, \btheta_1)} \geq \max_{h \in \calH, \btheta \in \PMTheta} \left\{ \bm{\alpha} \cdot \bm{\Gamma}_{(h, \btheta)} \right\} - \epsilon.
\]
Next, we prove the inductive step.
Assume that the lemma holds for some $T$. Now, for all $(h, \btheta) \in \calH \times \PMTheta$, we have
\begin{align*}
     \sum_{t=1}^{T+1}  \sprod{\btheta_{t+1}}{h_{t+1}(\x_t)\cdot\blf_t} + \bm{\alpha}\cdot \bm{\Gamma}_{(h_1, \btheta_1)} &= \sum_{t=1}^{T}  \sprod{\btheta_{t+1}}{h_{t+1}(\x_t)\cdot\blf_t} + \bm{\alpha}\cdot \bm{\Gamma}_{(h_1, \btheta_1)} \\ 
     & \qquad \qquad \qquad + \sprod{\btheta_{T+2}}{h_{T+2}(\x_{T+1})\cdot \blf_{T+1}} \\
     &\geq \sum_{t=1}^{T}  \sprod{\btheta_{T+2}}{h_{T+2}(\x_t)\cdot\blf_t} + \bm{\alpha}\cdot \bm{\Gamma}_{(h_{T+2}, \btheta_{T+2})} - \epsilon\cdot (T+1) \\
     &\qquad +  \sprod{\btheta_{T+2}}{h_{T+2}(\x_{T+1})\cdot\blf_{T+1}} \\
     &= \sum_{t=1}^{T+1}  \sprod{\btheta_{T+2}}{h_{T+2}(\x_t)\cdot\blf_t} + \bm{\alpha}\cdot \bm{\Gamma}_{(h_{T+2}, \btheta_{T+2})} -\epsilon \cdot (T+1) \\
     &\geq \sum_{t=1}^{T+1}  \sprod{\btheta}{h(\x_t)\cdot\blf_t} + \bm{\alpha}\cdot \bm{\Gamma}_{(h, \btheta)} - \epsilon \cdot (T+2)
\end{align*}
where the first inequality follows by applying the induction hypothesis and considering the pair $(h_{T+2}, \btheta_{T+2})$, and the final inequality follows by the $\epsilon$-approximate optimality of the oracle at time $T+2$.
This completes the proof of the lemma.
\end{proof}

Next, we characterize the term $T_2$ in~\eqref{eq:regret-h-decomp}.
The term $T_2$ is essentially a \emph{stability error term}, which measures the cumulative effect of the difference in the decisions $(h_{t+1},\btheta_{t+1})$ and $(h_t,\btheta_t)$ on reward.
Intuitively, a stable algorithm will result in similar or slowly varying decisions over time and therefore a small stability error term.
The following lemma characterizes the stability error term under the condition that the perturbation matrix $\bm{\Gamma}$ is $\delta$-admissible.
The proof is identical to the proof of the stability lemma (Lemma 2.4) presented in \cite{dudik2020oracle}. We include it here using our $\olp$-specific notation for completeness.

\begin{lemma}[Stability Lemma]\label{lem:perturbation}
Suppose we run \Cref{alg:olp-gftpl} with a $\delta$-admissible matrix $\bm{\Gamma} \in [0, 1]^{(|\calH| \times 2^M) \times N}$ and random vector $\bm{\alpha} = (\alpha_1, \cdots, \alpha_N)$ such that each $\alpha_i$ is drawn independently from Unif$[0, \sqrt{T}]$. Then, we have 
\[
T_2 \,\, \leq \,\, 8TNB\cdot (1+\delta^{-1})\cdot \left( \frac{B+\epsilon}{\sqrt{T}} \right).
\]
\end{lemma}

\begin{proof}
The main idea underlying the proof is to show that for every $t \in [T]$, we have $\bm{\Gamma}_{(h_t, \btheta_t)} = \bm{\Gamma}_{(h_{t+1}, \btheta_{t+1})}$ with high probability. 
We show at the end of the proof how this leads to the desired bound on $T_2$.

For the bulk of the proof, we fix $t \in [T]$ and $j \in [N]$. We will show that with high probability $\bm{\Gamma}_{(h_t, \btheta_t), j} = \bm{\Gamma}_{(h_{t+1}, \btheta_{t+1}), j}$. 
Specifically, we will show that
\begin{equation}
    \pr(\bm{\Gamma}_{(h_t, \btheta_t), j} \neq \bm{\Gamma}_{(h_{t+1}, \btheta_{t+1}), j}) \leq 4 \cdot (1+\delta^{-1})\cdot\left( \frac{B+\epsilon}{\sqrt{T}} \right).
\end{equation}
By $\delta$-admissibility, the $j^{\text{th}}$ column can take on at most $(1+\delta^{-1})$ distinct values; let $V$ denote this set. 
For any value $v \in V$, let $(h_v, \btheta_v) \in \calH \times \PMTheta$ be any action that maximizes the payoff (per Equation~\eqref{eq:oracle-ptm}) among those whose $\bm{\Gamma}$ entry in the corresponding row and $j^{\text{th}}$ column equals $v$. Formally, 
\begin{align*}
    (h_v, \btheta_v) &\in \argmax_{(h, \btheta): \bm{\Gamma}_{(h, \btheta), j}=v}\left[ \sum_{s=1}^{t-1} \sprod{\btheta}{h(\x_s)\cdot\blf_s} + \bm{\alpha}\cdot \bm{\Gamma}_{(h, \btheta)}\right] \\
    &= \argmax_{(h, \btheta): \bm{\Gamma}_{(h, \btheta), j}=v}\left[ \sum_{s=1}^{t-1} \sprod{\btheta}{h(\x_s)\cdot\blf_s} + \bm{\alpha}\cdot \bm{\Gamma}_{(h, \btheta)} - \alpha_j \cdot v\right]
\end{align*}
where the equality holds because the inclusion of the term $\alpha_j \cdot v$ does not change the identity of the maximizing action(s).
Furthermore, for any $v' \in V$, we define
\[
\Delta_{vv'} = \left( \sum_{s=1}^{t-1} \sprod{\btheta_v}{h_v(\x_s)\cdot\blf_s} + \bm{\alpha}\cdot \bm{\Gamma}_{(h_v, \btheta_v)} - \alpha_j \cdot v \right) - \left( \sum_{s=1}^{t-1} \sprod{\btheta_{v'}}{h_{v'}(\x_s)\cdot\blf_s} + \bm{\alpha}\cdot \bm{\Gamma}_{(h_{v'}, \btheta_{v'})} - \alpha_j \cdot v' \right)
\]
Note that, as defined, $(h_v, \btheta_v)$ and $\Delta_{v, v'}$ are independent of $\alpha_j$. In what follows, we derive a range of values for $\alpha_j$ that result in $\bm{\Gamma}_{(h_t, \btheta_t), j} \neq \bm{\Gamma}_{(h_{t+1}, \btheta_{t+1}), j}$.
Suppose that $\bm{\Gamma}_{(h_t, \btheta_t), j} = v$. 
First, observe that by the $\epsilon$-optimality of $(h_{t}, \btheta_t)$, we get
\begin{align*}
    &\sum_{s=1}^{t-1} \sprod{\btheta_t}{h_t(\x_s)\cdot\blf_s} + \bm{\alpha}\cdot \bm{\Gamma}_{(h_t, \btheta_t)} \geq \sum_{s=1}^{t-1} \sprod{\btheta_{v'}}{h_{v'}(\x_s)\cdot\blf_s} + \bm{\alpha}\cdot \bm{\Gamma}_{(h_{v'}, \btheta_{v'})} - \epsilon \\
    & \qquad \qquad =  \sum_{s=1}^{t-1} \sprod{\btheta_{v'}}{h_{v'}(\x_s)\cdot\blf_s} + \bm{\alpha}\cdot \bm{\Gamma}_{(h_{v'}, \btheta_{v'})} - \alpha_j \cdot v' + \alpha_j \cdot v - \epsilon + \alpha_j \cdot (v' - v),
\end{align*}
which, on re-arranging and by the definition of $(h_v, \btheta_v)$, gives
\begin{align}
    \Delta_{vv'} \geq \alpha_j \cdot (v'-v) - \epsilon \text{ for all } v' \in V. \label{eq:delta-lb}
\end{align}
Now, suppose that $\bm{\Gamma}_{h_{t+1}, \btheta_{t+1}, j} = v' \neq v$. Again, by the definition of $(h_{v'}, \btheta_{v'})$ and the $\epsilon$-optimality of $(h_{t+1}, \btheta_{t+1})$, we get
\begin{align*}
   \sum_{s=1}^{t-1} \sprod{\btheta_{v'}}{h_{v'}(\x_s) \cdot \blf_s} &+ \sprod{\btheta_{v'}}{h_{v'}(\x_t) \cdot \blf_t} + \bm{\alpha}\cdot \bm{\Gamma}_{(h_{v'}, \btheta_{v'})} \\
   & \qquad \qquad \geq \sum_{s=1}^{t-1} \sprod{\btheta_{t+1}}{h_{t+1}(\x_s) \cdot \blf_s} + \sprod{\btheta_{t+1}}{h_{t+1}(\x_t) \cdot \blf_t} + \bm{\alpha}\cdot \bm{\Gamma}_{(h_{t+1}, \btheta_{t+1})} \\
   & \qquad \qquad \geq \sum_{s=1}^{t-1} \sprod{\btheta_{v}}{h_{v}(\x_s) \cdot \blf_s} + \sprod{\btheta_{v}}{h_{v}(\x_t) \cdot \blf_t} + \bm{\alpha}\cdot \bm{\Gamma}_{(h_{v}, \btheta_{v})} - \epsilon.
\end{align*}
Adding $\alpha_j \cdot (v'-v)$ and re-arranging gives
\begin{align}
    \Delta_{vv'} \leq \alpha_j\cdot(v'-v) + \sprod{\btheta_{v'}}{h_{v'}(\x_t)\cdot \blf_t} -  \sprod{\btheta_{v}}{h_{v}(\x_t)\cdot \blf_t} + \epsilon \leq \alpha_j \cdot (v'-v) + 2B + \epsilon\notag
\end{align}
where the final inequality uses $\calH \sse \calH_B$.
If $v' > v$, then we need
\[
\alpha_j \geq \frac{\Delta_{vv'}-2B - \epsilon}{v'-v} \geq \min_{w \in V: w > v} \frac{\Delta_{vw}-2B - \epsilon}{w-v}.
\]
Let $w \in V$ also be the minimizer above; then, combined with \eqref{eq:delta-lb}, we get $\alpha_j \in [\frac{\Delta_{vw}-2B - \epsilon}{w-v}, \frac{\Delta_{vw} + \epsilon}{w - v}]$.
On the other hand, if $v' < v$, we need
\[
\alpha_j \leq \frac{\Delta_{vv'}+2B + \epsilon}{v - v'} \leq \max_{u \in V: u < v} \frac{\Delta_{vu}+2B + \epsilon}{v-u}.
\]
Let $u \in V$ be the maximizer above; in this case, combined with \eqref{eq:delta-lb}, we need $\alpha_j \in [\frac{\Delta_{vu} - \epsilon}{v-u}, \frac{\Delta_{vu}+2B + \epsilon}{v-u}]$. Now, fix a realization $\{\alpha_k\}_{k \neq j}$. As a consequence of the above discussion, and because $\{\alpha_j\}_{j \in [N]}$ are mutually independent, we get
\begin{align*}
    &\pr\left(\bm{\Gamma}_{(h_t, \btheta_t), j} \neq \bm{\Gamma}_{(h_{t+1}, \btheta_{t+1}), j} \mid \alpha_k, k\neq j\right) \\
    & \qquad \qquad \leq \pr\left( \exists v \in V: \alpha_j \in \left[\frac{\Delta_{vw}-2B - \epsilon}{w-v}, \frac{\Delta_{vw} + \epsilon}{w - v}\right] \text{ or }  \alpha_j \in \left[\frac{\Delta_{vu} - \epsilon}{v-u}, \frac{\Delta_{vu}+2B + \epsilon}{v-u}\right] \mid \alpha_k, k\neq j  \right) \\
    & \qquad \qquad \leq \sum_{v \in V}\pr\left(\alpha_j \in \left[\frac{\Delta_{vw}-2B - \epsilon}{w-v}, \frac{\Delta_{vw} + \epsilon}{w - v}\right] \text{ or }  \alpha_j \in \left[\frac{\Delta_{vu} - \epsilon}{v-u}, \frac{\Delta_{vu}+2B + \epsilon}{v-u}\right] \mid \alpha_k, k\neq j \right) \\
    & \qquad \qquad \leq 2 \cdot |V| \cdot \frac{2B+2\epsilon}{\delta} \cdot \frac{1}{\sqrt{T}} \leq 4 \cdot (1+\delta^{-1})\cdot\left( \frac{B+\epsilon}{\sqrt{T}}  \right),
\end{align*}
where the penultimate inequality uses that $w - v \geq \delta$, $v - u \geq \delta$ and that $\alpha_j \sim \text{Unif}[0, \sqrt{T}]$. 
Since this bound holds for all realizations of $\{\alpha_k\}_{k \neq j}$, it also holds unconditionally.

We now outline the proof of the upper bound on $T_2$.
Since, from the $\delta$-admissibility condition, the rows of $\bm{\Gamma}$ are distinct, $\bm{\Gamma}_{(h_t, \btheta_t)} = \bm{\Gamma}_{(h_{t+1}, \btheta_{t+1})}$ implies that $(h_t, \btheta_t) = (h_{t+1}, \btheta_{t+1})$. 
This, in turn, implies that $\sprod{\btheta_{t+1}}{h_{t+1}(\x_t)\cdot \blf_t} - \sprod{\btheta_{t}}{h_{t}(\x_t)\cdot\blf_t} = 0$.
On the other hand, if $(h_t,\btheta_t) \neq (h_{t+1},\btheta_{t+1})$, we have $\sprod{\btheta_{t+1}}{h(\x_{t+1})\cdot\blf_{t+1}} - \sprod{\btheta_{t}}{h(\x_{t})\cdot\blf_{t}} \leq 2B$.
The proof is completed by the law of total probability, taking a union bound over all $j \in [N]$ and summing over $t \in [T]$.

\end{proof}

Plugging our upper bounds on $T_1$ and $T_2$ into~\eqref{eq:regret-h-decomp} gives us
\begin{align}
    R_T(h) \,\, &\leq \,\, \E\left[\bm{\alpha}\cdot \left(\bm{\Gamma}_{(h_1, \btheta_1)} - \bm{\Gamma}_{(h, \btheta_h^*)}\right)\right] + \epsilon\cdot(T+1) + 8TNB\cdot (1+\delta^{-1})\cdot \left( \frac{B+\epsilon}{\sqrt{T}} \right) \notag  \\
    &\leq \,\, \E\left[\bm{\alpha}\cdot \bm{\Gamma}_{(h_1, \btheta_1)} \right] + \epsilon\cdot(T+1) + 8TNB\cdot (1+\delta^{-1})\cdot \left( \frac{B+\epsilon}{\sqrt{T}} \right) \notag \\
    &\leq \,\, N\sqrt{T} + \epsilon\cdot(T+1) + \frac{32B^2N\sqrt{T}}{\delta} = O\left(\frac{B^2N\sqrt{T}}{\delta}\right) \notag 
\end{align}
Above, the first inequality uses the fact that $\bm{\alpha} \succeq \mathbf{0}$ and $\bm{\Gamma} \in [0,1]^{(|\calH| \times 2^M) \times N}$; the second inequality uses the fact that $\bm{\Gamma} \in [0, 1]^{(|\calH| \times 2^M) \times N}$ and each $\alpha_i$ is drawn independently from Unif$[0, \sqrt{T}]$, and the final equality uses the assumption that $\epsilon = 1/\sqrt{T}$.

Finally, we argue that if $\Gamma$ satisfies implementability with complexity $D$, then \Cref{alg:olp-gftpl} with $\epsilon = 1/\sqrt{T}$ can be implemented in time $poly(N, D, T)$. This directly follows by noting that for any $\bm{\alpha}$, we have a set of coefficients $\{\kappa_{j}\}_{j=1}^D$ and vectors $\{\widehat{\blf}_j\}_{j= 1}^D$ (which may depend on $\bm{\alpha}$) such that
\[
\sum_{s=1}^{t-1} \sprod{\btheta}{h(\x_s)\cdot \blf_s} + \bm{\alpha} \cdot \bm{\Gamma}_{(h,\btheta)} = \sum_{s=1}^{t-1} \sprod{\btheta}{h(\x_s)\cdot \blf_s} +  \sum_{j = 1}^{D} \kappa_j \cdot \sprod{\btheta}{h(\x_j) \cdot \widehat{\blf}_j}. 
\]
This is exactly in the form of our offline oracle (\Cref{def:oracle}).
Thus, the runtime of Algorithm~\ref{alg:olp-gftpl}, treating as black boxes the step of sampling from the uniform distribution and the oracle~\eqref{eq:oracle} is $O(T^2 + TND)$. 
This completes the proof of \Cref{thm:oracle-efficient-full}.

\bibliographystyle{plainnat}
\bibliography{main}

\end{document}